\documentclass{article}
\usepackage{bbding}
\usepackage{colortbl}  
\usepackage{multirow}
\usepackage{multicol}
\usepackage{titlesec}
\usepackage{titletoc}
\usepackage{makeidx}
\usepackage{titletoc}
\usepackage{mathrsfs}  
\usepackage{amsthm}
\usepackage{amsmath}
\allowdisplaybreaks
\usepackage{microtype}
\usepackage{graphicx}
\usepackage{subfigure}
\usepackage{booktabs} 

\usepackage{hyperref}

\usepackage[accepted]{icml2025}

\usepackage{amsmath}
\usepackage{amssymb}
\usepackage{mathtools}
\usepackage{amsthm}

\usepackage[capitalize,noabbrev]{cleveref}

\theoremstyle{plain}
\newtheorem{theorem}{Theorem}[section]

\theoremstyle{definition}
\newtheorem{definition}[theorem]{Definition}

\theoremstyle{remark}

\usepackage[textsize=tiny]{todonotes}

\icmltitlerunning{Diffusion-based Adversarial Purification from the Perspective of the Frequency Domain}

\begin{document}

\twocolumn[
\icmltitle{Diffusion-based Adversarial Purification\\ from the Perspective of the Frequency Domain }




\begin{icmlauthorlist}
\icmlauthor{Gaozheng Pei}{eece-ucas}
\icmlauthor{Ke Ma}{eece-ucas}
\icmlauthor{Yingfei Sun}{eece-ucas}
\icmlauthor{Qianqian Xu}{ict}
\icmlauthor{Qingming Huang}{cs-ucas,ict,bdmkm}

\end{icmlauthorlist}

\icmlaffiliation{eece-ucas}{School of Electronic, Electrical and Communication Engineering, UCAS, Beijing.}
\icmlaffiliation{ict}{Key Laboratory of Intelligent Information Processing, Institute of Computing Technology, CAS, Beijing.}
\icmlaffiliation{cs-ucas}{School of Computer Science and Technology, UCAS, Beijing.}
\icmlaffiliation{bdmkm}{Key Laboratory of Big Data Mining and Knowledge Management, UCAS, Beijing.}

\icmlcorrespondingauthor{Ke Ma}{make@ucas.ac.cn}
\icmlcorrespondingauthor{Qingming Huang}{qmhuang@ucas.ac.cn}

\icmlkeywords{Machine Learning, ICML}

\vskip 0.3in]



\printAffiliationsAndNotice{}  

\begin{abstract}

The diffusion-based adversarial purification methods attempt to drown adversarial perturbations into a part of isotropic noise through the forward process, and then recover the clean images through the reverse process. Due to the lack of distribution information about adversarial perturbations in the pixel domain, it is often unavoidable to damage normal semantics. We turn to the frequency domain perspective, decomposing the image into amplitude spectrum and phase spectrum. We find that for both spectra, the damage caused by adversarial perturbations tends to increase monotonically with frequency. This means that we can extract the content and structural information of the original clean sample from the frequency components that are less damaged. Meanwhile, theoretical analysis indicates that existing purification methods indiscriminately damage all frequency components, leading to excessive damage to the image. Therefore, we propose a purification method that can eliminate adversarial perturbations while maximizing the preservation of the content and structure of the original image. Specifically, at each time step during the reverse process, for the amplitude spectrum, we replace the low-frequency components of the estimated image's amplitude spectrum with the corresponding parts of the adversarial image.
For the phase spectrum, we project the phase of the estimated image into a designated range of the adversarial image's phase spectrum, focusing on the low frequencies. Empirical evidence from extensive experiments demonstrates that our method significantly outperforms most current defense methods. Code is available at \url{https://github.com/GaozhengPei/FreqPure}.
\end{abstract}
\section{Introduction}

Adversarial purification \cite{diffpure,GDMAP,robust-evaluation,Contrastive,mimic,lorid} is a data preprocessing technique aimed at transforming adversarial images back into their original clean images during the testing phase. Compared to adversarial training \cite{gosch2024adversarial,wang2024revisiting,singh2024revisiting}, it offers advantages such as decoupling training and testing, and strong generalization. The main challenge faced by adversarial purification is how to eliminate adversarial perturbations while preserving the original semantic information as much as possible. Therefore, it is essential to explore how adversarial perturbations damage the images.

Current research \cite{chen2022rethinking,wang2020high,Rethinking-Frequency,maiya2021frequency,Rethinking-Frequency,maiya2023unifying} studies the distribution of adversarial perturbations in the frequency domain, but they lack quantitative analysis and typically do not distinguish between amplitude spectrum and phase spectrum. Due to the different distribution characteristics of amplitude spectrum and phase spectrum, in this paper, we further decompose images into amplitude spectrum and phase spectrum. We statistically analyze the variations in amplitude spectrum and phase spectrum of original clean images and adversarial images across multiple models, attack methods, and different perturbation radii (more experimental results and details in the Appendix \ref{More}). From Figure \ref{change} we find that for both the amplitude spectrum and the phase spectrum, the degradation caused by adversarial perturbations exhibits an approximately monotonically increasing trend with frequency. 

Existing methods of diffusion-based adversarial purification attempt to drown adversarial perturbations as part of isotropic noise through the forward process, and then recover clean images through the reverse process.  However, we theoretically prove that this strategy will destroy all frequency components in both the amplitude spectrum and phase spectrum and this destruction becomes increasingly severe as the time-step $t$ increases. The empirical  findings above suggest that if we want to restore adversarial images to their original clean images, we should minimize the disruption to the low-frequency amplitude spectrum and phase spectrum to ensure consistency and preserve the inherent characteristics of the original images. Therefore, we believe that current diffusion-based adversarial purification methods cause excessive damage to the semantic information of the input images. 

\begin{figure}[!t]
    \centering
    \includegraphics[width=0.5\textwidth]{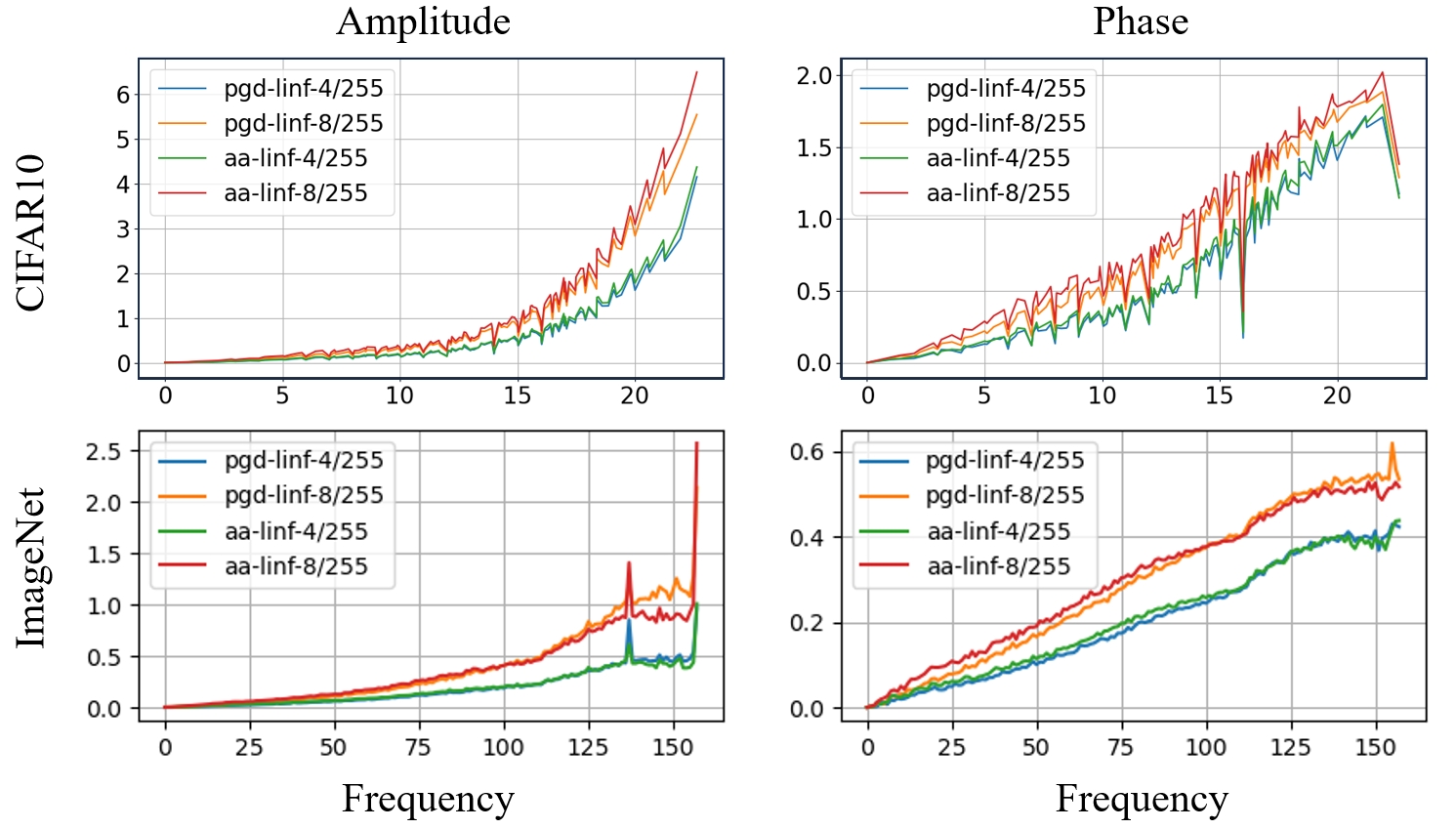}
    \caption{We decompose the image into the amplitude spectrum (left) and the phase spectrum (right), and calculate the differences between the adversarial images and the original images, respectively. The damage caused by adversarial perturbations tends to increase monotonically with frequency for both spectra.}
    \label{change}
\end{figure}

Motivated by experimental findings and theoretical analysis, we propose a novel adversarial purification method that can remove adversarial perturbations from adversarial images while minimizing destruction to the image. Specifically, at each time-step during the reverse process, we decompose the predicted estimated clean image  into its amplitude spectrum and phase spectrum. For amplitude spectrum, since low-frequency components of the adversarial image are almost unaffected, we replace the low-frequency part of the estimated image's amplitude spectrum with the low-frequency part of the adversarial image's amplitude spectrum. For phase spectrum,  we project the estimated image's low-frequency phase spectrum onto a certain range of the adversarial image's low-frequency phase spectrum. This is because the low-frequency phase spectrum is less affected by adversarial perturbations, allowing us to extract coarse-grained structural information from the image while gradually aligning it with the low-frequency phase information of natural images. Our proposed method selectively retains low-frequency phase and amplitude spectrum information, which not only preserves some structural and content information of the image but also provides prior information for the restoration of high-frequency details.
 \\
Overall, the contribution of this paper is as follows:
\begin{enumerate}
    \item We decompose the image into amplitude spectrum and phase spectrum, and explore how adversarial perturbation disturbs the original image from the perspective of the frequency domain.
    \item We theoretically demonstrate that current diffusion-based purification methods excessively destroy the amplitude spectrum and phase spectrum of input images.
    \item  Our proposed method retains the original structural information and content while eliminating adversarial perturbations by selectively preserving the amplitude spectrum and phase spectrum at each time-step.
    \item  Extensive experiments show that our method outperforms other methods by a promising improvement against various adversarial attacks.
\end{enumerate}
\section{Related Work}
Adversarial purification is a technique designed to eliminate adversarial perturbations from input images before classification, ensuring more robust model performance. These purification approaches can be categorized into two primary paradigms: training-based methods that necessitate dataset preparation, and diffusion-based techniques that offer a training-agnostic solution, capable of operating without direct access to original training data.
\subsection{Training-Based Adversarial Purification}
\cite{pixeldefend} empirically demonstrates that adversarial examples predominantly exist in the low probability areas of the training distribution and aims to redirect them back towards this distribution. \cite{defensegan} initially models the distribution of clean images and subsequently seeks the nearest clean sample to the adversarial example during inference.  \cite{SSP} generates perturbed images using a self-supervised perturbation attack that disrupts the deep perceptual features and projects back the perturbed images close to the perceptual space of clean images. \cite{Invariant} introduces a method to learn generalizable invariant features across various attacks using an encoder, engaging in a zero-sum game to reconstruct the original image with a decoder.  \cite{ATOP} combines adversarial training and purification techniques via employing random transforms to disrupt adversarial perturbations and fine-tunes a purifier model using adversarial loss.  \cite{RFGSM} leverages the phenomenon of FGSM robust overfitting to enhance the robustness of deep neural networks against unknown adversarial attacks. Nonetheless, these approaches necessitate training on the training dataset, which is time-intensive and lacks generalizability.
\subsection{Diffusion-based Adversarial Purification}
\cite{score-based} shows that an energy-based model trained with denoising score-matching can quickly purify attacked images within a few steps. \cite{diffpure} begins by adding a small amount of noise to the images and then recovers the clean image through a reverse generative process \cite{GDMAP} suggests using the adversarial image as a reference during the reverse process, which helps ensure that the purified image aligns closely with the original clean image. \cite{robust-evaluation} proposed a new gradient estimation method and introduced a stepwise noise scheduling strategy to enhance the effectiveness of the current purification methods. \cite{mimic} proposes a method that reduces the negative impact of adversarial perturbations by mimicking the generative process with clean images as input. \cite{Contrastive} designs the forward process with the proper amount of Gaussian noise added and the reverse process with the gradient of contrastive loss as the guidance of diffusion models for adversarial purification. However, our theoretical analysis demonstrates that these methods excessively disrupt the semantic information of the images.
\begin{figure*}[htbp]
    \centering
    \includegraphics[width=1.0\textwidth]{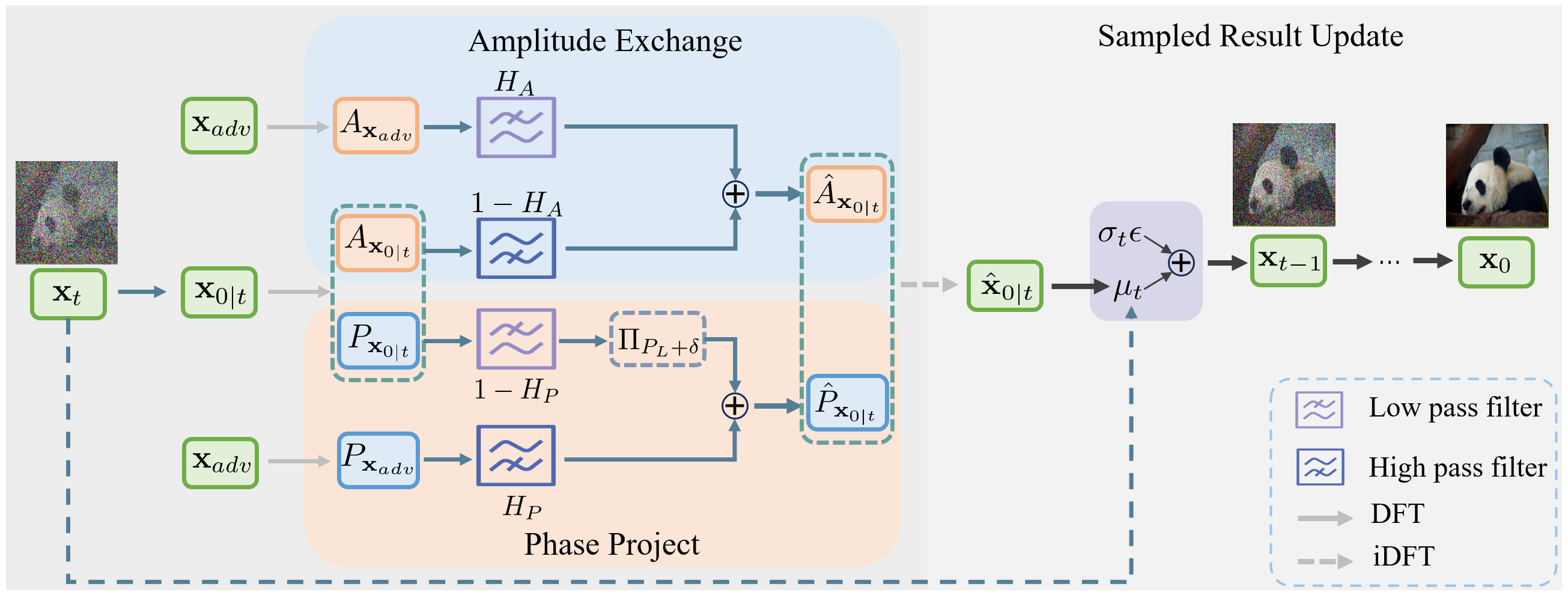}
    \caption{Pipeline of our method. The core of our method is to preserve, at each time-step of the reverse process, the amplitude and phase spectrum information of the original clean samples extracted from adversarial images as a prior. This ensures the retention of the original image's content and structural information while also providing guidance for the restoration of high-frequency details.}
    \label{method}
\end{figure*}
\section{Theoretical Study}
\label{Theoretical}
We perform discrete Fourier transform (DFT) on the input image $\mathbf{x}_0$ and the noisy image $\mathbf{x}_t$ obtained using forward process \cite{DDPM} at time-step $t$ as follows:
\begin{equation}
\begin{aligned}
\mathbf{x}_0(u,v)&=DFT(\mathbf{x}_0)=|\mathbf{x}_0(u,v)|e^{i\phi_{\mathbf{x}_0}(u,v)},\\    \mathbf{x}_t(u,v)&=DFT(\mathbf{x}_t)=|\mathbf{x}_t(u,v)|e^{i\phi_{\mathbf{x}_t}(u,v)},
\end{aligned}
\end{equation}
where $u$ and $v$ are coordinates at frequency domain. $|\mathbf{x}_0(u,v)|$ and $|\mathbf{x}_t(u,v)|$ denote the amplitude spectrum, $\phi_{\mathbf{x}_0}(u,v)$ and $\phi_{\mathbf{x}_t}(u,v)$ denote the phase spectrum.
\begin{definition}
(Difference of  amplitude) Given the  amplitude $|\mathbf{x}_0(u,v)|$ of the input image $\mathbf{x}_0$ and the amplitude $|\mathbf{x}_t(u,v)|$ of the noisy image $\mathbf{x}_t$, respectively. We definite the difference between there two amplitude of the images at arbitrary coordinate $(u,v)$ are as follows:
\begin{equation}
        \Delta A_t(u,v)=|\mathbf{x}_t(u,v)|-|\mathbf{x}_0(u,v)|.
\end{equation}
This value represents the degree of variation in the image content.
\end{definition}

\begin{theorem}
\label{amplitude-theorem}
    (Proof in Appendix \ref{proof-3-2}) The variance of the difference of amplitude at time-step $t$ between clean image $\mathbf{x}_0$ and noisy image $\mathbf{x}_t$ at arbitrary coordinates $(u,v)$ at frequency domain is as follows:
    \begin{equation}
        Var(\Delta A_t(u,v))\approx\frac{1-\overline{\alpha}_t}{2} -\frac{(1-\overline{\alpha}_t)^2}{16|\mathbf{x}_0(u,v)|\overline{\alpha}_t}.
    \end{equation}
    The RHS is monotonically increasing with respect to $t$, This means that as $t$ increases, the amplitude spectrum of the original image at arbitrary coordinate $(u,v)$  is increasingly disrupted by noise.
\end{theorem}
\begin{definition}
(Difference of phase) Given the phase $\phi_{\mathbf{x}_0}$ of the input image $\mathbf{x}_0$ and the phase $\phi_{\mathbf{x}_t}$ of the noisy image $\mathbf{x}_t$, respectively. We definite the difference between there two phase of the images at arbitrary coordinate $(u,v)$ are as follows:
\begin{equation}
        \Delta\theta_t(u,v)=\phi_{\mathbf{x}_t}(u,v)-\phi_{\mathbf{x}_0}(u,v).
\end{equation}
This value represents the degree of variation in the shape and structure of the image.
\end{definition}
\begin{theorem}
\label{phase-theorem}
    (Proof in Appendix \ref{proof-3-4}) The variance of the first-order approximation of the difference of phase between clean image $\mathbf{x}_0$ and noisy image $\mathbf{x}_t$ at arbitrary coordinates $(u,v)$ at frequency domain is as follows:
    \begin{equation}
         Var(\Delta\theta_t(u,v))=\frac{1}{\sqrt{1-\frac{1}{SNR_t^2(u,v)}}}-1,
    \end{equation}
    where signal to noise ratio (SNR) at time-step $t$ is defined as follows:
    \begin{equation}
        SNR_t(u,v)=\frac{\sqrt{\overline{\alpha}_t}|\mathbf{x}_0(u,v)|}{\sqrt{1-\overline{\alpha}_t}|\mathbf{\epsilon}(u,v)|}.
    \end{equation}

    Obviously, $SNR_t(u,v)$ decreases monotonically with $t$, so the variance  of the difference of phase $Var(\Delta\theta(u,v))$  increases monotonically with $t$. This means that as $t$ increases, the phase spectrum of the original image is increasingly disrupted by noise.
\end{theorem}
\textcolor{black}{\textbf{Remark 3.5.} From Theorem \ref{amplitude-theorem} and Theorem \ref{phase-theorem}, we can conclude that all frequency components of the image are disrupted by the forward process, and the degree of disruption increases monotonically with $t$ ($\overline{\alpha_t}$ decreases monotonically with $t$ ). From Figure \ref{change}, we can find that as $t$ increases, while the adversarial perturbations are drowned out into part of  isotropic noise, the normal content and structural information contained in the amplitude and phase spectra of low frequencies will also be disrupted by the forward process, which means that  the existing method may result in excessive disruption of the input images.}

\section{Methodology}
To minimize the damage to semantic information while eliminating adversarial perturbations, we propose a novel adversarial purification method named FreqPure from the perspective of the frequency domain.
The core of the FreqPure method lies in selectively preserving low-frequency phase and amplitude spectrum information which are less affected by adversarial perturbations. Our approach not only maintains some structural features and content of the original image during the reverse process but also provides valuable prior constraints for the reconstruction of high-frequency information in the image.\\ 
Given an adversarial example \( \mathbf{x}_{adv} \) at time-step \( t=0 \), \textit{i.e.}, \( \mathbf{x}_0 = \mathbf{x}_{adv}\in \mathbb{R}^{H\times W \times C} \). We first diffuse it according to the forward process from \( t=0 \) to \( t=t^* \). Then, during the reverse diffusion process at each time-step, we aim to yield clean intermediate states for refinement. Following \cite{wang2023zeroshot}, to obtain ``clean" samples, we estimate \( \mathbf{x}_0 \) from \( \mathbf{x}_t \) and the predicted noise \( \mathcal{Z}_\theta(\mathbf{x}_t, t) \). We denote the estimated image \( \mathbf{x}_0 \) at time-step \( t \) as \( \mathbf{x}_{0|t} \), which can be formulated as 
\begin{equation}
    \mathbf{x}_{0|t} = \frac{1}{\sqrt{\alpha_t}} \left( \mathbf{x}_t - \mathcal{Z}_\theta(\mathbf{x}_t, t) \sqrt{1 - \bar{\alpha}_t} \right),
\end{equation}
where $ \mathcal{Z}_\theta$ is a neural network used by DDPM\cite{DDPM}  to predict the noise $\epsilon$ for each time-step $t$, i.e., $\epsilon_t= \mathcal{Z}_\theta(\mathbf{x}_t, t)$. Then, we perform DFT on both the input image $\mathbf{x}_0$ and the estimated image $\mathbf{x}_{0|t}$ to decompose them into the
amplitude spectrum and phase spectrum as follows:
\begin{equation}(A_{\mathbf{x}_0}(u,v),P_{\mathbf{x}_0}(u,v))=DFT(\mathbf{x}_0),
\end{equation}
\begin{equation}(A_{\mathbf{x}_{0|t}}(u,v),P_{\mathbf{x}_{0|t}}(u,v))=DFT(\mathbf{x}_{0|t}),
\end{equation}
where $A_{\mathbf{x}_0}(u,v),A_{\mathbf{x}_{0|t}}(u,v)$,$P_{\mathbf{x}_0}(u,v)$ and $P_{\mathbf{x}_{0|t}}(u,v)$ represent the amplitude spectrum and phase spectrum of input image $\mathbf{x}_0$ and estimated image $\mathbf{x}_{0|t}$, respectively. $(u,v)$ denotes the coordinate at the frequency domain. The amplitude spectrum reflects the energy distribution of various frequency components in the image. The phase spectrum contains the structural and shape information of the image.

\subsection{Amplitude Spectrum Exchange}
Experimental reveals that low-frequency amplitude spectrum components demonstrate significant robustness to adversarial perturbations, being almost unaffected by such disturbances. Additionally, natural signals (such as images) generally exhibit low-pass characteristics, meaning that the low-frequency power spectrum components are relatively large. Therefore, we opt to retain this portion of the amplitude spectrum and replace the low-frequency amplitude spectrum components of the estimated image $\mathbf{x}_{0|t}$ accordingly at each time-step. We first construct a filter $H_A(u,v)$ for amplitude spectrum as follows:
\begin{equation}
\label{filter}
    H_A(u,v)=\left\{
    \begin{array}{cc}
    1,     & D(u,v)<D_A \\
    0,     & D(u,v)>D_A
    \end{array}
\right.,
\end{equation}
where $D_A$ is hyper-parameter, $D(u,v)$ represents the distance from point $(u,v)$ to the center of the $H\times W$ frequency rectangle in the frequency domain, which is also the magnitude of the frequency. The formula is as follows:
\begin{equation}
    D(u,v)=[(u-H/2)^2+(v-W/2)^2]^{\frac{1}{2}},
\end{equation}
With the filter $H_A$ defined above, we can replace the low-frequency components of the estimated image's amplitude spectrum with the low-frequency components of the input sample's amplitude spectrum for each channel (Color images are typically composed of three channels: RGB) as follows:
\begin{equation}
    \hat{A}_{\mathbf{x}_{0|t}} = A_{\mathbf{x}_{0|t}}\times(1-H_A)+A_{\mathbf{x}_0}\times H_A,
\end{equation}
$\hat{A}_{\mathbf{x}_{0|t}}$ represents the updated amplitude spectrum of the estimated image $\mathbf{x}_{0|t}$.

\subsection{Phase Spectrum Projection}
Different from the amplitude spectrum, the phase spectrum is affected by adversarial perturbations at all frequency components. Directly retaining the low-frequency phase spectrum, which is less disturbed, will preserve the adversarial perturbations while also affecting the restoration of the high-frequency phase spectrum. Therefore, we choose to project the estimated image's low-frequency phase spectrum into a certain range of the input image's low-frequency phase spectrum. First, We construct a filter $H_P$ for phase spectrum which is the same as \eqref{filter} but with a different hyper-parameter $D_P$. Then, we update the phase spectrum of the sampled result as follows:
\begin{equation}
    \hat{P}_{\mathbf{x}_{0|t}}=\Pi_{P_L+\delta}(\underbrace{P_{\mathbf{x}_0}\times H_P}_{P_L})+P_{\mathbf{x}_{0|t}}\times(1-H_P),
\end{equation}
where $\Pi$ is the projection operation. $P_L$ is the low-frequency phase spectrum and $\delta$  denotes the range within which we allow variations in the low-frequency phase spectrum of the estimated image $\mathbf{x}_{0|t}$.\\
This strategy can benefit from two aspects: 1) We can extract coarse-grained low-frequency structural information while allowing it to vary within a certain range, enabling it to gradually align with the low-frequency structural information of natural images. 2) The coarse-grained low-frequency structural information can provide prior guidance for the recovery of high-frequency information.\\
Overall, we aim to extract and retain information from the original clean image by focusing on frequencies that are less affected by adversarial perturbations, thereby maximizing the preservation of the original image's structure and content, while also providing a correct prior for the restoration of high-frequency information.
\subsection{Next State Generation}
With the updated amplitude spectrum $\hat{A}_{\mathbf{x}_{0|t}}$ and phase spectrum $\hat{P}_{\mathbf{x}_{0|t}}$. We combine $\hat{A}_{\mathbf{x}_{0|t}}$ and $\hat{P}_{\mathbf{x}_{0|t}}$ obtain their representation in the time domain through the inverse discrete Fourier transformation ($iDFT$) as follows:
\begin{equation}
    \hat{\mathbf{x}}_{0|t} = iDFT(\hat{A}_{\mathbf{x}_{0|t}},\hat{P}_{\mathbf{x}_{0|t}}).
\end{equation}
As suggested in \cite{wang2023zeroshot}, the next state $\mathbf{x}_{t-1}$ can be sampled from a joint distribution, which is formulated as:
\begin{equation}  
p_\theta (\mathbf{x}_{t-1} | \mathbf{x}_t, \hat{\mathbf{x}}_{0|t}) = \mathcal{N}(\mu_t(\mathbf{x}_t, \hat{\mathbf{x}}_{0|t}); \sigma_t^2 I) ,
\end{equation}  
where  $\mu_t(\mathbf{x}_t, \hat{\mathbf{x}}_{0|t}) = \frac{\sqrt{\alpha_{t-1} \beta_t}}{1 - \alpha_t} \hat{\mathbf{x}}_{0|t} + \frac{\sqrt{\alpha_t(1 - \bar{\alpha}_{t-1})}}{1 - \alpha_t} \mathbf{x}_t  $
and  $ \sigma_t^2 = \frac{1 - \bar{\alpha}_{t-1}}{1 - \alpha_t} \beta_t.  $ By using the low-frequency priors of the phase spectrum and amplitude spectrum provided by the input image to guide the sampling process in each time-step $t$, we ultimately obtain the clean image \(\hat{x}_0\) with natural amplitude spectrum and phase spectrum. The complete algorithm process can refer to \ref{algorithm}.
\begin{algorithm}  
\label{algorithm}
\caption{Sampling Process}  
\begin{algorithmic}[1]  
\REQUIRE Sample $\mathbf{x}_{adv}$, timestep $t^*$.
    \STATE $\mathbf{x}_t \sim \mathcal{N}(0, I)$  
    \FOR{$t = t^*, \ldots, 1$}  
        \STATE $\mathbf{x}_{0|t} = \frac{1}{\sqrt{\alpha_t}} \left( \mathbf{x}_t - \mathcal{Z}_\theta(\mathbf{x}_t, t) \sqrt{1 - \bar{\alpha}_t} \right)$  
        \STATE $(A_{\mathbf{x}_{0|t}}, P_{\mathbf{x}_{0|t}}) = DFT(\mathbf{x}_{0|t}) $
        \STATE $(A_{\mathbf{x}_{adv}}, P_{\mathbf{x}_{adv}}) = DFT(\mathbf{x}_{adv}) $
        \STATE \textcolor{gray}{// Amplitude Spectrum Exchange}
        \STATE $    \hat{A}_{\mathbf{x}_{0|t}} = A_{\mathbf{x}_{0|t}}\times(1-H_A)+A_{\mathbf{x}_0}\times H_A$
        \STATE \textcolor{gray}{// Phase Spectrum Projection}
        \STATE $ \hat{P}_{\mathbf{x}_{0|t}}=\Pi_{P_L+\delta}({P_L})+P_{\mathbf{x}_{0|t}}\times(1-H_P)$
        \STATE \textcolor{gray}{// Next State Generation}
        \STATE $\hat{\mathbf{x}}_{0|t}=iDFT(\hat{A}_{\mathbf{x}_{0|t}},\hat{P}_{\mathbf{x}_{0|t}})$
        \STATE $\mathbf{x}_{t-1} \sim p(\mathbf{x}_{t-1} | \mathbf{x}_t, \hat{\mathbf{x}}_{0|t})$  

    \ENDFOR  
    \STATE \textbf{Output} $\mathbf{x}_0$  
\end{algorithmic}  
\end{algorithm}  
\begin{figure*}[!t]
    \centering
    \includegraphics[width=0.9\textwidth]{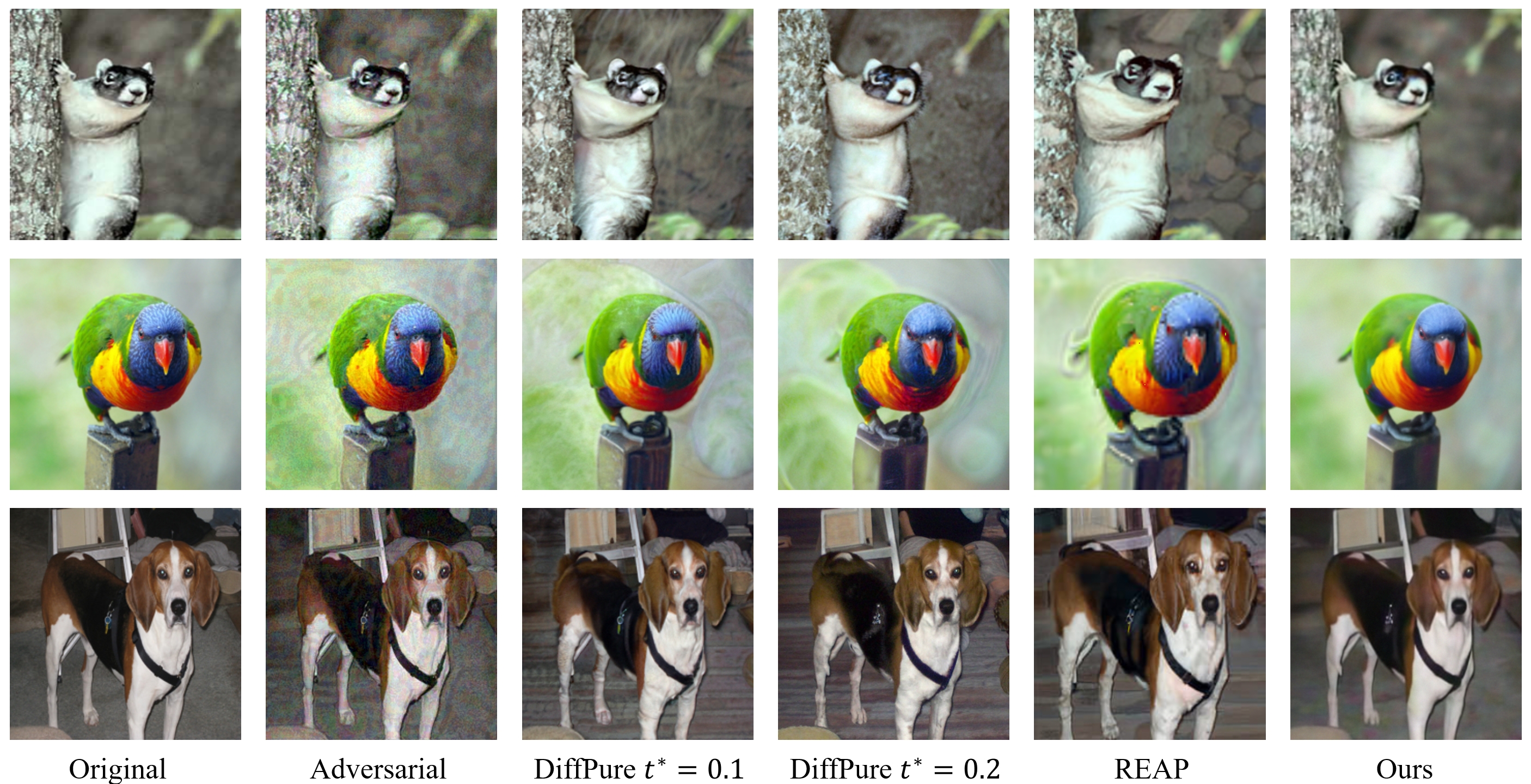}
    \caption{Visualization of origianl clean images , adversarial images and purified images. The images purified by our method are most similar to the origianl clean images.}
    \label{Visualization-main}
\end{figure*} 
\section{Experiment}
\subsection{Experimental Settings}
\textbf{Datasets and network architectures.}\quad Three datasets are utilized for evaluation: CIFAR-10 \cite{cifar10}, SVHN and ImageNet \cite{imagenet}. Our results are compared against several prominent defense techniques listed in the standardized benchmark RobustBench \cite{robustbench} for both CIFAR-10 and ImageNet, and we also examine various adversarial purification methods. For CIFAR-10, we employ two widely used classifier architectures: WideResNet-28-10 and WideResNet-70-16 \cite{wide}. In the case of SVHN, WideResNet-28-10 acts as the backbone, whereas ResNet-50 \cite{resnet} is utilized for ImageNet. \\
\textbf{Adversarial attack methods.}\quad We evaluate strong adaptive attacks \cite{sa1,sa2} against our approach and other adversarial purification methods.  The well-known AutoAttack \cite{autoattack} is implemented under $\ell_\infty$ and $\ell_2$ threat models.  Furthermore, the projected gradient descent (PGD) attack \cite{PGD} is assessed on our method, as suggested in \cite{robust-evaluation}. To account for the randomness introduced by the diffusion and denoising processes, Expectation Over Transformation (EOT) \cite{sa1} is adapted for these adaptive attacks. Additionally, we employ the BPDA+EOT \cite{hill2021stochastic} attack to facilitate a fair comparison with other adversarial purification methods. Lastly, following the recommendations of \cite{robust-evaluation}, a surrogate process is utilized to derive the gradient of the reverse process for white-box attacks.\\
\textbf{Pre-trained Models.}\quad We utilize the unconditional CIFAR-10 checkpoint of EDM supplied by NVIDIA \cite{chekpoint-cifar10} for the CIFAR-10 dataset. For the ImageNet experiments, we adopt the 256x256 unconditional diffusion checkpoint from the guided-diffusion library. The pre-trained classifier for CIFAR-10 is obtained from RobustBench \cite{robustbench}, whereas the classifier weights for ImageNet are sourced from the TorchVision library.   \\
\textbf{Evaluation metrics.} \quad To evaluate the effectiveness of defense methods, we employ two metrics: standard accuracy, which is calculated on clean images, and robust accuracy, assessed on adversarial examples. Given the significant computational expense associated with evaluating models against adaptive attacks, we randomly sample a fixed subset of 512 images from the test set for robust evaluation, consistent with \cite{diffpure,robust-evaluation,mimic,ATOP,Contrastive}. In all experiments, we report the mean and standard deviation across three runs to evaluate both standard and robust accuracy.\\
\textbf{Implementation details.} \quad We adhere to the configurations outlined in \cite{robust-evaluation}. Diffusion-based purification methods are evaluated using the PGD attack with 200 update iterations, while BPDA and AutoAttack are assessed with 100 update iterations, except for ImageNet, which utilizes 20 iterations. The number of EOT is set to 20, and the step size is 0.007. For randomized defenses, such as those in \cite{diffpure,robust-evaluation,mimic,Contrastive}, we employ the random version of AutoAttack, whereas the standard version is used for static defenses.

\subsection{Experimental Results}
\begin{table}[ht]
    \setlength\tabcolsep{4pt}
  \centering
  \caption{Standard and robust accuracy of different Adversarial Training (AT) and Adversarial Purification (AP) methods against PGD and AutoAttack $\ell_\infty (\mathbf{\epsilon} = 8/255)$ on CIFAR-10. $\textbf{}^*$ utilizes half number of iterations for the attack due to the high computational cost. $\textbf{}^{\dag}$ indicates the requirement of extra data. The result with an underline indicates the second highest.}
    \label{linf-cifar10-28-10}
\scriptsize
    \begin{tabular}{ccccc}
    \toprule
    \multirow{2}[0]{*}{ Type } & \multirow{2}[0]{*}{Method}  & \multirow{2}[0]{*}{Standard Acc.} & \multicolumn{2}{c}{Robust Acc.} \\
          &       &                & PGD   & AutoAttack \\
    \midrule
    \rowcolor{gray!20} \multicolumn{5}{c}{WideResNet-28-10}\\
    \midrule
        & \cite{improving}    & 88.54 & 65.93 & 63.38\\
    AT    &\cite{uncovering}$^{\dag}$   & 87.51 & 66.01 & 62.76\\
        &\cite{Robustness}
    & 88.62 & 64.95 &61.04\\
    \midrule
        &\cite{score-based} & \cellcolor[rgb]{ .827,  .898,  .961}85.66±0.51 & \cellcolor[rgb]{ .867,  .922,  .969}33.48±0.86 & \cellcolor[rgb]{ .863,  .922,  .969}59.53±0.87 \\
        & \cite{diffpure} & \cellcolor[rgb]{ .698,  .82,  .925}90.07±0.97 & \cellcolor[rgb]{ .635,  .78,  .91}\underline{56.84±0.59} & \cellcolor[rgb]{ .816,  .89,  .957}63.60±0.81 \\
        &\cite{robust-evaluation} & \cellcolor[rgb]{ .698,  .816,  .925}90.16±0.64 & \cellcolor[rgb]{ .647,  .784,  .914}55.82±0.59 & \cellcolor[rgb]{ .729,  .839,  .933}70.47±1.53 \\
    AP    &\cite{Contrastive} & \cellcolor[rgb]{ .659,  .796,  .918}\underline{91.41} & \cellcolor[rgb]{ .71,  .827,  .929}49.22$^{\ast}$ & \cellcolor[rgb]{ .651,  .788,  .914}\underline{77.08} \\
        &\cite{lorid}  & \cellcolor[rgb]{ .867,  .922,  .969}84.20 &   -    & \cellcolor[rgb]{ .867,  .922,  .969}59.14 \\
        &\cite{ATOP} & \cellcolor[rgb]{ .682,  .808,  .922}90.62 &   -    & \cellcolor[rgb]{ .702,  .82,  .929}72.85 \\
        & Ours  & \cellcolor[rgb]{ .608,  .761,  .902}\textbf{92.19 ±0.33} & \cellcolor[rgb]{ .608,  .761,  .902}\textbf{59.39±0.79} & \cellcolor[rgb]{ .608,  .761,  .902}\textbf{77.35±2.14} \\

    \bottomrule
    \end{tabular}%
\end{table}%
\begin{table}[ht]
    \setlength\tabcolsep{4pt}
  \centering
  \caption{Standard and robust accuracy of different Adversarial Training (AT) and Adversarial Purification (AP) methods against PGD and AutoAttack $\ell_\infty (\mathbf{\epsilon} = 8/255)$ on CIFAR-10.  $\textbf{}^*$ The number of iterations for the attack is half that of the other methods for less computational overhead. $\textbf{}^{\dag}$ indicates the requirement of extra data. The result with an underline indicates the second highest.}
    \label{linf-cifar10-70-16}
\scriptsize
    \begin{tabular}{ccccc}

    \toprule
    \multirow{2}[0]{*}{Type } & \multirow{2}[0]{*}{Method}  & \multirow{2}[0]{*}{Standard Acc.} & \multicolumn{2}{c}{Robust Acc.} \\
          &       &                & PGD   & AutoAttack \\
    \midrule
     \rowcolor{gray!20} \multicolumn{5}{c}{WideResNet-70-16}\\
    \midrule
        &\cite{rebuffi2021fixing}$^{\dag}$   & 92.22 & 69.97 &66.56\\
    AT    &\cite{uncovering}$^{\dag}$    & 91.10 & 68.66 &66.10\\
        &\cite{improving}   & 88.75 & 69.03 &65.87\\
    \midrule

        &\cite{score-based} & \cellcolor[rgb]{ .804,  .882,  .953}86.76±1.15 & \cellcolor[rgb]{ .867,  .922,  .969}37.11±1.35 & \cellcolor[rgb]{ .851,  .914,  .965}60.86±0.56 \\
        &\cite{diffpure} & \cellcolor[rgb]{ .698,  .816,  .925}90.43±0.60 & \cellcolor[rgb]{ .71,  .824,  .929}51.13±0.87 & \cellcolor[rgb]{ .8,  .882,  .953}66.06±1.17 \\
        &\cite{robust-evaluation} & \cellcolor[rgb]{ .694,  .816,  .925}90.53±0.1 & \cellcolor[rgb]{ .643,  .784,  .914}\underline{56.88±1.06} & \cellcolor[rgb]{ .757,  .855,  .941}70.31±0.62 \\
        & \cite{Contrastive} & \cellcolor[rgb]{ .62,  .769,  .906}\textbf{92.97} & \cellcolor[rgb]{ .733,  .839,  .937}48.83$^{\ast}$ & \cellcolor[rgb]{ .667,  .8,  .918}\textbf{79.10} \\
    AP    &\cite{lorid} & \cellcolor[rgb]{ .867,  .922,  .969}84.60 &   -    & \cellcolor[rgb]{ .796,  .878,  .953}66.40 \\
        &\cite{lorid} & \cellcolor[rgb]{ .8,  .882,  .953}86.90 &   -    & \cellcolor[rgb]{ .867,  .922,  .969}59.20 \\
        &\cite{ATOP} & \cellcolor[rgb]{ .651,  .788,  .914}91.99 &    -   & \cellcolor[rgb]{ .694,  .816,  .925}76.37 \\
        & Ours  & \cellcolor[rgb]{ .608,  .761,  .902}\underline{92.52±0.53} & \cellcolor[rgb]{ .608,  .761,  .902}\textbf{62.50±2.73} & \cellcolor[rgb]{ .608,  .761,  .902}\underline{78.13±1.95} \\

    \bottomrule
    \end{tabular}%
\end{table}%
\begin{table}[ht]
\scriptsize
    \setlength\tabcolsep{4pt}
  \centering
  \caption{Standard and robust accuracy against PGD and AutoAttack $\ell_2 (\mathbf{\epsilon} = 0.5)$ on CIFAR-10. Adversarial Training (AT) and Adversarial Purification (AP) methods are evaluated. $\textbf{}^\ast$The number of iterations for the attack is half that of the other methods for less computational overhead. $\textbf{}^{\dag}$ indicates the requirement of extra data. $\textbf{}^\ddag$ adopts  WideResNet-34-10 as the backbone, with the same width but more layers than the default one. The result with an underline indicates the second highest.}
\label{l2-cifar10-28-10}  
    \begin{tabular}{ccccc}
    \toprule
    \multirow{2}[0]{*}{Type } & \multirow{2}[0]{*}{Method}  & \multirow{2}[0]{*}{Standard Acc.} & \multicolumn{2}{c}{Robust Acc.} \\
          &          & & PGD   & AutoAttack \\
    \midrule
    \rowcolor{gray!20} \multicolumn{5}{c}{WideResNet-28-10}\\
    \midrule
        &\cite{rebuffi2021fixing}$^{\dag}$ & 91.79 & 85.05 &78.80 \\
    AT    &\cite{augustin2020adversarial}$^{\ddag}$ & 93.96 & 86.14&78.79 \\
        &\cite{Robustness}$^{\ddag}$ 
    & 90.93 & 83.75&77.24 \\
    \midrule
        &\cite{score-based} & \cellcolor[rgb]{ .827,  .898,  .961}85.66±0.51 & \cellcolor[rgb]{ .867,  .922,  .969}73.32±0.76 & \cellcolor[rgb]{ .776,  .867,  .945}79.57±0.38 \\
        &\cite{diffpure} & \cellcolor[rgb]{ .675,  .804,  .922}91.41±1.00 & \cellcolor[rgb]{ .753,  .851,  .941}79.45±1.16 & \cellcolor[rgb]{ .741,  .847,  .937}81.7±0.84 \\
        &\cite{robust-evaluation}  & \cellcolor[rgb]{ .706,  .824,  .929}90.16±0.64 & \cellcolor[rgb]{ .675,  .804,  .922}83.59±0.88 & \cellcolor[rgb]{ .667,  .8,  .918}\underline{86.48±0.38} \\
        &\cite{Contrastive}  & \cellcolor[rgb]{ .627,  .773,  .91}\underline{91.41}  & \cellcolor[rgb]{ .627,  .773,  .91}\underline{86.13}$^{\ast}$ & \cellcolor[rgb]{ .753,  .851,  .941}80.92 \\
    AP    & \cite{lorid} & \cellcolor[rgb]{ .863,  .922,  .969}84.40 &   -    & \cellcolor[rgb]{ .8,  .882,  .953}77.90 \\
        &\cite{lorid} & \cellcolor[rgb]{ .867,  .922,  .969}84.20 &  -     & \cellcolor[rgb]{ .867,  .922,  .969}73.60 \\
        &\cite{ATOP} & \cellcolor[rgb]{ .694,  .816,  .925}90.62 &   -    & \cellcolor[rgb]{ .761,  .859,  .941}80.47 \\
       & Ours  & \cellcolor[rgb]{ .608,  .761,  .902}\textbf{92.19 ±0.33 } & \cellcolor[rgb]{ .608,  .761,  .902}\textbf{87.89±1.17} & \cellcolor[rgb]{ .608,  .761,  .902}\textbf{89.06±0.43} \\
    \bottomrule
    \end{tabular}%
\end{table}%
\begin{table}[ht]
\scriptsize
\setlength\tabcolsep{4pt}
  \centering
  \caption{Standard and robust accuracy against PGD and AutoAttack $\ell_2 (\mathbf{\epsilon} = 0.5)$ on CIFAR-10. Adversarial Training (AT) and Adversarial Purification (AP) methods are evaluated. ($\textbf{}^\ast$The number of iterations for the attack is half that of other methods for less computational overhead. $\textbf{}^{\dag}$ methods need extra data.) The result with an underline indicates the second highest.}
\label{l2-cifar10-70-16}  
    \begin{tabular}{ccccc}
    \toprule
    \multirow{2}[0]{*}{Type } & \multirow{2}[0]{*}{Method}  & \multirow{2}[0]{*}{Standard Acc.} & \multicolumn{2}{c}{Robust Acc.} \\
          &          & & PGD   & AutoAttack \\

    \midrule
    \rowcolor{gray!20} \multicolumn{5}{c}{WideResNet-70-16}\\
    \midrule
            &\cite{rebuffi2021fixing}$^{\dag}$ & 95.74 & 89.62&82.32 \\
    AT    &\cite{uncovering}$^{\dag}$  & 94.74 & 88.18 & 80.53\\
        &\cite{rebuffi2021fixing} & 92.41 & 86.24 & 80.42\\
    \midrule
        & \cite{score-based} & \cellcolor[rgb]{ .867,  .922,  .969}86.76±1.15 & \cellcolor[rgb]{ .867,  .922,  .969}75.666±1.29 & \cellcolor[rgb]{ .867,  .922,  .969}80.43±0.42 \\
        &\cite{diffpure} & \cellcolor[rgb]{ .659,  .792,  .918}92.15±0.72 & \cellcolor[rgb]{ .71,  .827,  .929}82.97±1.38 & \cellcolor[rgb]{ .808,  .886,  .957}83.06±1.27 \\
        &\cite{robust-evaluation}  & \cellcolor[rgb]{ .722,  .831,  .933}90.53±0.14 & \cellcolor[rgb]{ .694,  .816,  .925}83.75±0.99 & \cellcolor[rgb]{ .749,  .851,  .941}\underline{85.59±0.61} \\
    AP    & \cite{Contrastive}  & \cellcolor[rgb]{ .624,  .773,  .91}92.97 & \cellcolor[rgb]{ .682,  .808,  .922}\underline{84.37}$^{\ast}$ & \cellcolor[rgb]{ .808,  .886,  .957}83.01 \\
        &\cite{ATOP} & \cellcolor[rgb]{ .663,  .796,  .918}91.99 &    -   & \cellcolor[rgb]{ .847,  .91,  .965}81.35 \\
        & Ours  & \cellcolor[rgb]{ .608,  .761,  .902}\underline{92.52±0.53} & \cellcolor[rgb]{ .608,  .761,  .902}\textbf{87.89±1.17} & \cellcolor[rgb]{ .608,  .761,  .902}\textbf{89.84±1.56} \\
    \bottomrule
    \end{tabular}%
\end{table}%

\begin{table}[ht]
\scriptsize
\setlength\tabcolsep{4pt} 
  \centering
  \caption{Standard and robust accuracy against BPDA+EOT $\ell_\infty(\mathbf{\epsilon} = 8/255)$ on CIFAR-10. The result with an underline indicates the second highest.}
  \label{bpda}
    \begin{tabular}{ccccr}
    \toprule

    \multicolumn{2}{c}{\multirow{2}[0]{*}{Method}} & \multirow{2}[0]{*}{Purification} & \multicolumn{2}{c}{Accuracy} \\
    \multicolumn{2}{c}{} &       & Standard & \multicolumn{1}{c}{Robust} \\
    \midrule
     \rowcolor{gray!20} \multicolumn{5}{c}{WideResNet-28-10}\\
    \midrule
    \multicolumn{2}{c}{\cite{pixeldefend}} & Gibbs Update & 95.00    & \multicolumn{1}{c}{9.00} \\
    \multicolumn{2}{c}{\cite{menet}} & Mask+Recon & 94.00    & \multicolumn{1}{c}{15.00} \\
    \multicolumn{2}{c}{ \cite{hill2021stochastic}} & EBM+LD & 84.12 & \multicolumn{1}{c}{54.90} \\
    \midrule
    \multicolumn{2}{c}{\cite{score-based}} & DSM+LD &\cellcolor[rgb]{ .867,  .922,  .969} 85.66±0.51 & \cellcolor[rgb]{ .867,  .922,  .969}66.91±1.75\\
    \multicolumn{2}{c}{\cite{diffpure}} & Diffusion &\cellcolor[rgb]{ .714,  .827,  .929} 90.07±0.97 &  \cellcolor[rgb]{ .643,  .784,  .914}81.45±1.51 \\
    \multicolumn{2}{c}{ \cite{Contrastive}} & Diffusion &\cellcolor[rgb]{ .667,  .8,  .918} \underline{91.37±1.21} &  \cellcolor[rgb]{ .643,  .784,  .914}85.51±0.81 \\
    \multicolumn{2}{c}{\cite{GDMAP}} & Diffusion &\cellcolor[rgb]{ .718,  .831,  .933} 89.96±0.40 &\cellcolor[rgb]{ .765,  .859,  .945} 75.59±1.26 \\
    \multicolumn{2}{c}{ \cite{mimic}} & Diffusion &\cellcolor[rgb]{ .722,  .831,  .933} 89.88±0.35 &\cellcolor[rgb]{ .631,  .776,  .91} \underline{88.43±0.83} \\
    \multicolumn{2}{c}{\cite{robust-evaluation}} & Diffusion &\cellcolor[rgb]{ .71,  .827,  .929} 90.16±0.64 & \cellcolor[rgb]{ .612,  .765,  .906} 88.40±0.88 \\
    \multicolumn{2}{c}{Ours} & Diffusion & \cellcolor[rgb]{ .608,  .761,  .902}\textbf{ 92.19±0.79}&\cellcolor[rgb]{ .608,  .761,  .902}\textbf{ 89.07±0.79 } \\
    \bottomrule
    \end{tabular}%
\end{table}%

\begin{table}[ht]
\scriptsize
\setlength\tabcolsep{4pt}
  \centering
  \caption{Standard and robust accuracy against PGD $\ell_\infty (\mathbf{\epsilon} = 4/255)$ on ImageNet. ResNet-50 is used as the classifier. The result with an underline indicates the second highest.}
    \label{linf-imagenet}
    \begin{tabular}{cccc}
    \toprule
    \multirow{2}[0]{*}{Type} & \multirow{2}[0]{*}{Method} & \multicolumn{2}{c}{Accuracy} \\
          &       & Standard & Robust \\
    \midrule
    \rowcolor{gray!20} \multicolumn{4}{c}{ResNet-50}\\
    \midrule
        & \cite{transferbetter} & 63.86 & 39.11 \\
    AT    &  \cite{engstrom2019robustness}    & 62.42 & 33.20 \\
        & Wong et al. \cite{fastbetter} & 53.83 & 28.04 \\
    \midrule
        & ($t^*=0.2$)  \cite{diffpure} & \cellcolor[rgb]{ .62,  .769,  .906}\underline{73.96} & \cellcolor[rgb]{ .867,  .922,  .969}40.63 \\
        &  ($t^*=0.3$)  \cite{diffpure} & \cellcolor[rgb]{ .643,  .784,  .914}72.85 & \cellcolor[rgb]{ .769,  .859,  .945}48.24 \\
        &  ($t^*=0.4$)  \cite{diffpure} & \cellcolor[rgb]{ .867,  .922,  .969}61.71 & \cellcolor[rgb]{ .855,  .914,  .969}41.67 \\
        & \cite{Contrastive}  & \cellcolor[rgb]{ .694,  .816,  .925}70.41 & \cellcolor[rgb]{ .855,  .914,  .969}41.70 \\
    AP    & \cite{robust-evaluation} & \cellcolor[rgb]{ .671,  .8,  .922}71.42 & \cellcolor[rgb]{ .788,  .875,  .949}46.59 \\
        & \cite{mimic} & \cellcolor[rgb]{ .859,  .918,  .969}62.25 & \cellcolor[rgb]{ .729,  .835,  .933}51.14 \\
        &  \cite{lorid} & \cellcolor[rgb]{ .62,  .769,  .906}\textbf{73.98} & \cellcolor[rgb]{ .655,  .792,  .918}\underline{56.54} \\
        & Ours  & \cellcolor[rgb]{ .608,  .761,  .902}71.88 & \cellcolor[rgb]{ .608,  .761,  .902}\textbf{59.77}  \\
    \bottomrule
    \end{tabular}%
\end{table}%
\begin{figure}[ht]
    \centering
    \includegraphics[width=0.45\textwidth]{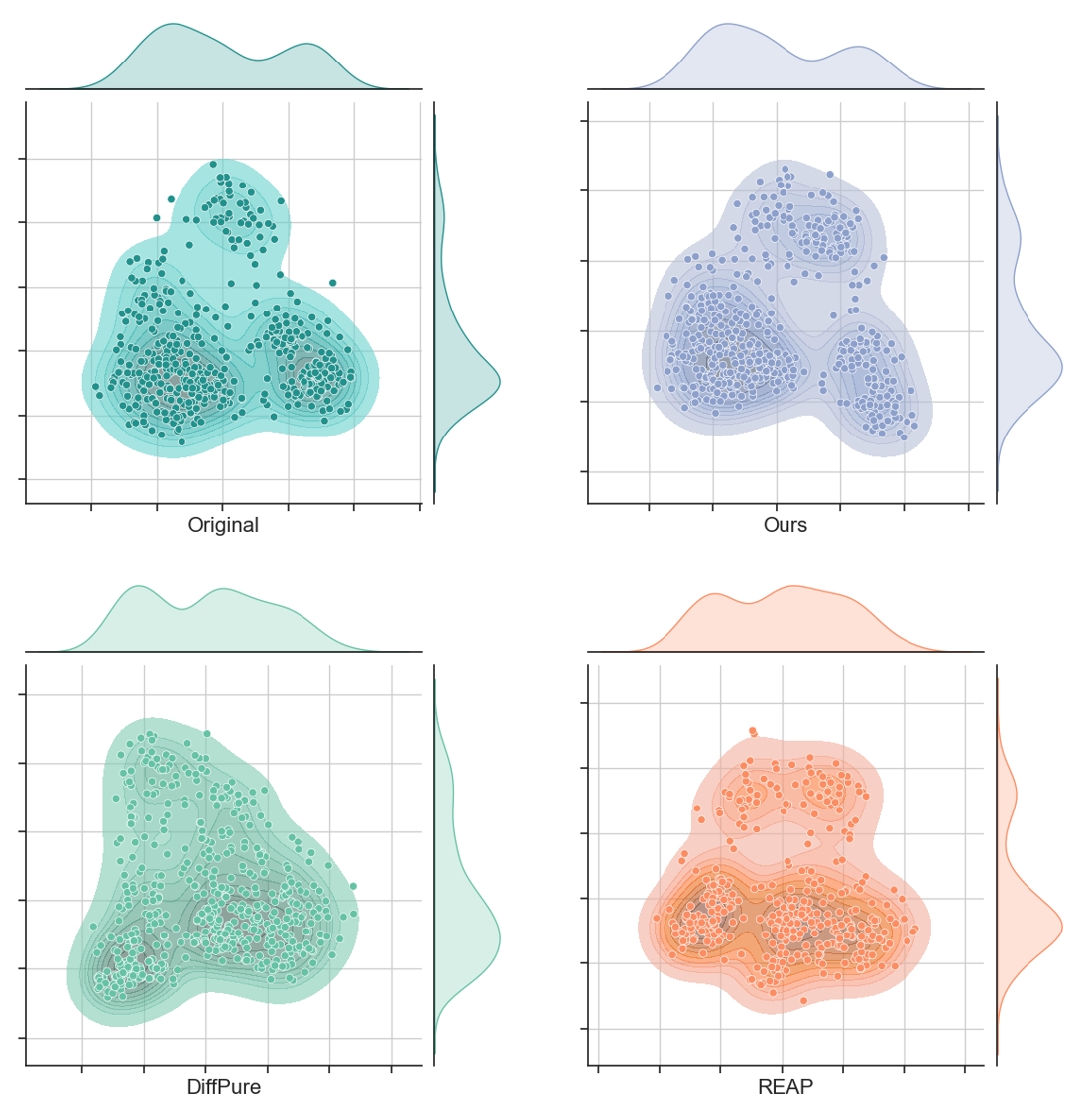}
    \caption{Joint distribution of the original images and purified images. The distributions of the purified images by our method and the original images are the most similar.}
    \label{distribution}
\end{figure}
\begin{table}[htbp]
  \centering
  \small
      \setlength\tabcolsep{2pt}
  \caption{To measure the similarity between the purified image and the clean image, we calculated the DINO similarity and CLIP similarity between them.}
    \begin{tabular}{cccccc}
    \toprule
          &       & DiffPure & REAP  & CGAP  & Ours \\
    \midrule
    \multirow{2}[0]{*}{DINO} & Standard & 0.887 & 0.860  &  0.879     & \textbf{0.917} \\
          & Robust & 0.820  & 0.805 &   0.822    & \textbf{0.877} \\
    \midrule
    \multirow{2}[0]{*}{CLIP} & Standard &  0.961     &  0.952     &  0.962     & \textbf{0.972} \\
          & Robust & 0.949      &  0.946     &  0.954     &\textbf{0.956} \\
    \bottomrule
    \end{tabular}%
  \label{similarity}%
\end{table}%
\textbf{CIFAR10}\quad We perform extensive experiments on the CIFAR-10 dataset, comparing our method with other approaches using two model architectures: WideResNet-28-10 and WideResNet-70-16. Robust accuracy is evaluated under three types of attacks: BPDA+EOT, PGD, and AutoAttack. For \cite{improving,uncovering,Robustness}, they typically construct adversarial examples based on other attack methods for adversarial training. Focusing on the $\ell_{\infty}$ attack, as illustrated in Table \ref{linf-cifar10-28-10}, our method demonstrates a significant advantage over other baselines when using WideResNet-28-10 as the backbone. Our method not only improves the standard accuracy metric by 0.78\% but also enhances robust accuracy under PGD and AutoAttack attacks by 2.55\% and 0.27\%, respectively. When WideResNet-70-16 is used as the backbone, from table \ref{linf-cifar10-70-16} we can see that though our method decreases the standard accuracy metric by 0.45\%, our method enhances robust accuracy under PGD and AutoAttack attacks by 5.62\% and 1.76\%, respectively. We also assess the accuracy of our method and other baselines against the $\ell_2$ attack. From Table \ref{l2-cifar10-28-10}, we observe that,  When WideResNet-28-10 is used as the backbone, Our method improves standard accuracy by 0.78\%, and robust accuracy increases by 1.86\% and 2.68\%, respectively. As shown in Table \ref{l2-cifar10-70-16} when WideResNet-70-16 is the backbone, our method outperforms other baselines, though decreasing by 0.45\% in standard accuracy and by 3.62\% and 4.25\% in robust accuracy under PGD and AutoAttack, respectively. Additionally, we apply the BPDA+EOT attack, which approximates differentiability. As shown in Table \ref{bpda}, Our method leads by 0.82\% in Standard Accuracy and 0.64\% in Robust Accuracy. To measure the degree of similarity between the purified samples and the original clean samples, We plot the distribution map of the purified images. As shown in Figure \ref{distribution}, our method yields a distribution of purified images that is the most similar to the original images when compared to other methods. In addition to qualitative analysis, we also conduct some quantitative analyses. Here, We opt to utilize measures like DINO similarity \cite{dino,dinov2} and CLIP similarity \cite{clip}, which calculate the cosine similarity between embeddings extracted from two images. From Table \ref{similarity}, Our method achieves the highest similarity score, indicating that the images purified by our method are the most akin to the original clean images. This is attributed to our strategy of appropriately preserving the inherent content and structural information of the images during the purification process.
Overall, our method outperforms others, demonstrating its effectiveness in preserving semantic information while eliminating adversarial perturbations.\\
\textbf{ImageNet}\quad We evaluate ResNet50 as the backbone of the ImageNet dataset under PGD attacks, consistent with \cite{diffpure,robust-evaluation,Contrastive}. From Table \ref{linf-imagenet}, our method outperforms competitors by 0.53\% in standard accuracy and 3.52\% in robust accuracy, significantly exceeding the baseline. Additionally, the method from \cite{diffpure} achieves a standard accuracy of 73.96\% and a robust accuracy of 40.63\% when $t^* = 0.2$. For $t^* = 0.3$,  standard accuracy decreases to 72.85\%, while robust accuracy increases to 48.24\%. However, at $t^* = 0.4$, both standard accuracy and robust accuracy decline to 61.71\% and 41.67\%, respectively. This demonstrates that traditional methods increasingly destroy semantic information as $t$ increases, which aligns with our theoretical proof. In addition, we also visualize the purified images. From Figure \ref{Visualization-main}, we observe that compared to DiffPure and REAP, our method achieves the best restoration effect. This is attributed to the prior guidance provided by extracting the low-frequency phase spectrum and amplitude spectrum with smaller perturbations during the image restoration process, which allows for effective recovery of high-frequency details. Our method decreases standard accuracy  by 2.1\%, but improves robust accuracy by 3.23\%, respectively. With average accuracy improves by 0.55\%. In general, these experiments can demonstrate the effectiveness of our approach.\\
\subsection{Ablation Study}
To demonstrate the effectiveness of our method, we conduct comprehensive ablation experiments. We test the standard accuracy and robust accuracy. We divide our method into several parts: the first part is amplitude spectrum exchange (ASE), and the other part is phase spectrum project (PSP). The backbone we chose is WideResNet-28-10, and the attack method is AutoAttack with $\ell_{\infty}$ and $\epsilon=8/255$.
\\ After removing ASE and PSP, the method becomes the same as \cite{robust-evaluation}. However, due to the different number of time-step used in our method, there are some differences in the results compared to the original method. All ablation studies are conducted under our predefined noise time-step to verify the effectiveness of each module.
From the table \ref{ablation}, we can see that removing either ASE or PSP will affect standard accuracy and robust accuracy, proving that our improvement is effective. 
\begin{table}[ht]
    \footnotesize
  \centering
  \caption{Standard accuracy and robust accuracy under different combinations. WideResNet28-10 servers as the backbone.}
    \begin{tabular}{ccccc}
    \toprule
     ASE  &  PSP    & Standard & Robust & Average  \\
    \midrule 
    \XSolidBrush&\XSolidBrush   & 91.79& 71.68& 81.73\\
    \Checkmark&\XSolidBrush   & 89.07&  77.35 &83.21\\
     \XSolidBrush&\Checkmark   & 91.41 & 76.17 &83.79\\
    \Checkmark&\Checkmark   &\textbf{92.19}  &\textbf{77.35} &\textbf{84.77} \\
       \bottomrule
    \end{tabular}%
  \label{ablation}%
\end{table}%

\section{Conclusion}
We analyze how adversarial perturbations disrupt the amplitude spectrum and phase spectrum from the perspective of the frequency domain. Additionally, we theoretically prove that current diffusion-based adversarial purification methods excessively damage images. Based on these, we propose refining the low-frequency amplitude spectrum and phase spectrum of the estimated image at each time-step during the reverse process. This approach not only preserves some structural information and image content but also guides the restoration of high-frequency components. However, both our approach and previous work still face challenges in accurately and quickly calculating gradients. This poses difficulties in evaluating the effectiveness of defense methods.

\section*{Impact Statement}
The work presented in this paper seeks to significantly advance the emerging field of machine learning security, specifically addressing vulnerabilities related to adversarial attacks. Recognizing that adversarial images pose severe threats to the integrity and reliability of machine learning models, we introduce a novel purification method leveraging diffusion models from a frequency domain perspective. Our approach uniquely exploits the frequency characteristics inherent to adversarial perturbations, enabling efficient and effective mitigation of malicious modifications to input images. This purification process enhances the robustness of model predictions, safeguarding them against a broad spectrum of adversarial strategies and ultimately improving their reliability and trustworthiness in real-world applications.
\section*{Acknowledge}
This work was supported in part by the National Science and Technology Major Project 2022ZD0119204, in part by National Natural Science Foundation of China: 62236008, 62441232, U21B2038, U23B2051, 62122075 and 62376257, in part by Youth Innovation Promotion Association CAS, in part by the Strategic Priority Research Program of the Chinese Academy of Sciences, Grant No. XDB0680201, in part by the Fundamental Research Funds for the Central Universities, in part by Xiaomi Young Talents Program.
\bibliography{example_paper}

@String(AAAI = {AAAI})

@inproceedings{
pixeldefend,
title={PixelDefend: Leveraging Generative Models to Understand and Defend against Adversarial Examples},
author={Yang Song and Taesup Kim and Sebastian Nowozin and Stefano Ermon and Nate Kushman},
booktitle={International Conference on Learning Representations},
year={2018},
  pages={24347--24356},
}

@inproceedings{
defensegan,
title={Defense-{GAN}: Protecting Classifiers Against Adversarial Attacks Using Generative Models},
author={Pouya Samangouei and Maya Kabkab and Rama Chellappa},
booktitle={International Conference on Learning Representations},
year={2018},
}

@INPROCEEDINGS{SSP,
  author={Naseer, Muzammal and Khan, Salman and Hayat, Munawar and Khan, Fahad Shahbaz and Porikli, Fatih},
  booktitle={IEEE/CVF Conference on Computer Vision and Pattern Recognition}, 
  title={A Self-supervised Approach for Adversarial Robustness}, 
  year={2020},
  pages={262--271},
}

@InProceedings{Invariant ,
  title = 	 {Towards Defending against Adversarial Examples via Attack-Invariant Features},
  author =       {Zhou, Dawei and Liu, Tongliang and Han, Bo and Wang, Nannan and Peng, Chunlei and Gao, Xinbo},
  booktitle = 	 {International Conference on Machine Learning},
  year = 	 {2021},
  pages={12835--12845},
}

@inproceedings{
ATOP,
title={Adversarial Training on Purification ({AT}oP): Advancing Both Robustness and Generalization},
author={Guang Lin and Chao Li and Jianhai Zhang and Toshihisa Tanaka and Qibin Zhao},
booktitle={International Conference on Learning Representations},
year={2024},
}

@InProceedings{RFGSM,
    author    = {Tang, Linyu and Zhang, Lei},
    title     = {Robust Overfitting Does Matter: Test-Time Adversarial Purification With FGSM},
    booktitle = {IEEE/CVF Conference on Computer Vision and Pattern Recognition},
    year      = {2024},
}

@InProceedings{score-based,
  title = 	 {Adversarial Purification with Score-based Generative Models},
  author =       {Yoon, Jongmin and Hwang, Sung Ju and Lee, Juho},
  booktitle = 	 {International Conference on Machine Learning},
  year = 	 {2021},
pages={12062--12072},
}

@InProceedings{diffpure,
  title = 	 {Diffusion Models for Adversarial Purification},
  author =       {Nie, Weili and Guo, Brandon and Huang, Yujia and Xiao, Chaowei and Vahdat, Arash and Anandkumar, Animashree},
  booktitle = 	 {International Conference on Machine Learning},
  year = 	 {2022},
  pages={16805--16827},
}

@ARTICLE{GDMAP,
  title={Guided diffusion model for adversarial purification},
  author={Wang, Jinyi and Lyu, Zhaoyang and Lin, Dahua and Dai, Bo and Fu, Hongfei},
  journal={arXiv preprint arXiv:2205.14969},
  year={2022}
}

@InProceedings{mimic,
    author    = {Song, Kaiyu and Lai, Hanjiang and Pan, Yan and Yin, Jian},
    title     = {MimicDiffusion: Purifying Adversarial Perturbation via Mimicking Clean Diffusion Model},
    booktitle = {IEEE/CVF Conference on Computer Vision and Pattern Recognition},
    year      = {2024},
  pages={24665--24674},
}

@inproceedings{
Contrastive,
title={Diffusion Models Demand Contrastive Guidance for Adversarial Purification to Advance},
author={Mingyuan Bai and Wei Huang and Tenghui Li and Andong Wang and Junbin Gao and Cesar F Caiafa and Qibin Zhao},
booktitle={International Conference on Machine Learning},
year={2024},
pages = 	 {2375--2391},
}

@article{DDPM,
  title={Denoising diffusion probabilistic models},
  author={Ho, Jonathan and Jain, Ajay and Abbeel, Pieter},
  journal={Annual Conference on Neural Information Processing Systems},
  year={2020},
  pages={6840--6851},
}

@inproceedings{resnet,
  title={Deep residual learning for image recognition},
  author={He, Kaiming and Zhang, Xiangyu and Ren, Shaoqing and Sun, Jian},
  booktitle={IEEE/CVF Conference on Computer Vision and Pattern Recognition},
  year={2016},
  pages={770--778},
}

@inproceedings{robustbench,
  title={Robustbench: a standardized adversarial robustness benchmark},
  author={Croce, Francesco and Andriushchenko, Maksym and Sehwag, Vikash and Debenedetti, Edoardo and Flammarion, Nicolas and Chiang, Mung and Mittal, Prateek and Hein, Matthias},
  booktitle={Annual Conference on Neural Information Processing Systems},
  year={2021}
}

@inproceedings{imagenet,
  title={Imagenet: A large-scale hierarchical image database},
  author={Deng, Jia and Dong, Wei and Socher, Richard and Li, Li-Jia and Li, Kai and Fei-Fei, Li},
  booktitle={IEEE/CVF Computer Vision and Pattern Recognition},
  year={2009},
  pages={248--255},
}

@article{cifar10,
  title={Learning multiple layers of features from tiny images},
  author={Krizhevsky, Alex and Hinton, Geoffrey},
  year={2009},
  publisher={Toronto, ON, Canada}
}

@inproceedings{wide,
  title={Wide Residual Networks},
  author={Zagoruyko, Sergey and Komodakis, Nikos},
  booktitle={British Machine Vision Conference},
  year={2016},
}

@inproceedings{robust-evaluation,
  title={Robust evaluation of diffusion-based adversarial purification},
  author={Lee, Minjong and Kim, Dongwoo},
  booktitle={IEEE/CVF International Conference on Computer Vision},
  year={2023},
  pages={134--144},
}

@inproceedings{
PGD,
title={Towards Deep Learning Models Resistant to Adversarial Attacks},
author={Aleksander Madry and Aleksandar Makelov and Ludwig Schmidt and Dimitris Tsipras and Adrian Vladu},
booktitle={International Conference on Learning Representations},
year={2018},
}

@inproceedings{autoattack,
  title={Reliable evaluation of adversarial robustness with an ensemble of diverse parameter-free attacks},
  author={Croce, Francesco and Hein, Matthias},
  booktitle={International Conference on Machine Learning},
  year={2020},
  pages={2206--2216},
}

@article{chekpoint-cifar10,
  title={Elucidating the design space of diffusion-based generative models},
  author={Karras, Tero and Aittala, Miika and Aila, Timo and Laine, Samuli},
  journal={Annual Conference on Neural Information Processing Systems},
  pages={26565--26577},
  year={2022}
}

@article{improving,
  title={Improving robustness using generated data},
  author={Gowal, Sven and Rebuffi, Sylvestre-Alvise and Wiles, Olivia and Stimberg, Florian and Calian, Dan Andrei and Mann, Timothy A},
  journal={Annual Conference on Neural Information Processing Systems},
  year={2021},
  pages={4218--4233},
}

@article{uncovering,
  title={Uncovering the limits of adversarial training against norm-bounded adversarial examples},
  author={Gowal, Sven and Qin, Chongli and Uesato, Jonathan and Mann, Timothy and Kohli, Pushmeet},
  journal={arXiv preprint arXiv:2010.03593},
  year={2020}
}

@inproceedings{Robustness,
  title={Robustness and accuracy could be reconcilable by (proper) definition},
  author={Pang, Tianyu and Lin, Min and Yang, Xiao and Zhu, Jun and Yan, Shuicheng},
  booktitle={International Conference on Machine Learning},
  year={2022},
  pages={17258--17277},
}

@inproceedings{lorid,
  title={Lorid: Low-rank iterative diffusion for adversarial purification},
  author={Zollicoffer, Geigh and Vu, Minh N and Nebgen, Ben and Castorena, Juan and Alexandrov, Boian and Bhattarai, Manish},
  booktitle={Proceedings of the AAAI Conference on Artificial Intelligence},
  pages={23081--23089},
  year={2025}
}

@article{rebuffi2021fixing,
  title={Data augmentation can improve robustness},
  author={Rebuffi, Sylvestre-Alvise and Gowal, Sven and Calian, Dan Andrei and Stimberg, Florian and Wiles, Olivia and Mann, Timothy A},
  journal={Annual Conference on Neural Information Processing Systems},
  pages={29935--29948},
  year={2021}
}

@inproceedings{augustin2020adversarial,
  title={Adversarial robustness on in-and out-distribution improves explainability},
  author={Augustin, Maximilian and Meinke, Alexander and Hein, Matthias},
  booktitle={European Conference on Computer Vision},
  year={2020},
  pages={228--245},
}

@article{transferbetter,
  title={Do adversarially robust imagenet models transfer better?},
  author={Salman, Hadi and Ilyas, Andrew and Engstrom, Logan and Kapoor, Ashish and Madry, Aleksander},
  journal={Annual Conference on Neural Information Processing Systems},
  year={2020},
  pages={3533--3545},
}

@article{engstrom2019robustness,
  title={Robustness (python library), 2019},
  author={Engstrom, Logan and Ilyas, Andrew and Salman, Hadi and Santurkar, Shibani and Tsipras, Dimitris},
  journal={URL https://github.com/MadryLab/robustness},
  year={2019}
}

@inproceedings{
fastbetter,
title={Fast is better than free: Revisiting adversarial training},
author={Eric Wong and Leslie Rice and J. Zico Kolter},
booktitle={International Conference on Learning Representations},
year={2020},

}

@inproceedings{menet,
  title={ME-Net: Towards Effective Adversarial Robustness with Matrix Estimation},
  author={Yang, Yuzhe and Zhang, Guo and Katabi, Dina and Xu, Zhi},
  booktitle={International Conference on Machine Learning},
  year={ 2019},
  pages={7025--7034},
}

@inproceedings{
hill2021stochastic,
title={Stochastic Security: Adversarial Defense Using Long-Run Dynamics of Energy-Based Models},
author={Mitch Hill and Jonathan Craig Mitchell and Song-Chun Zhu},
booktitle={International Conference on Learning Representations},
year={2021},
}

@article{gosch2024adversarial,
  title={Adversarial training for graph neural networks: Pitfalls, solutions, and new directions},
  author={Gosch, Lukas and Geisler, Simon and Sturm, Daniel and Charpentier, Bertrand and Z{\"u}gner, Daniel and G{\"u}nnemann, Stephan},
  journal={Annual Conference on Neural Information Processing Systems},
  pages={58088--58112},
  year={2023}
}

@inproceedings{wang2024revisiting,
  title={Revisiting Adversarial Training at Scale},
  author={Wang, Zeyu and Li, Xianhang and Zhu, Hongru and Xie, Cihang},
  booktitle={IEEE/CVF Conference on Computer Vision and Pattern Recognition},
  year={2024},
  pages={24675--24685},
}

@article{singh2024revisiting,
  title={Revisiting adversarial training for imagenet: Architectures, training and generalization across threat models},
  author={Singh, Naman Deep and Croce, Francesco and Hein, Matthias},
  journal={Annual Conference on Neural Information Processing Systems},
  year={2024}
}

@article{chen2022rethinking,
  title={Rethinking and improving robustness of convolutional neural networks: a shapley value-based approach in frequency domain},
  author={Chen, Yiting and Ren, Qibing and Yan, Junchi},
  journal={Annual Conference on Neural Information Processing Systems},
  year={2022},
  pages={324--337},
}

@inproceedings{wang2020high,
  title={High-frequency component helps explain the generalization of convolutional neural networks},
  author={Wang, Haohan and Wu, Xindi and Huang, Zeyi and Xing, Eric P},
  booktitle={IEEE/CVF Conference on Computer Vision and Pattern Recognition},
  year={2020},
  pages={8684--8694},
}

@inproceedings{Rethinking-Frequency,
author="Han, Sicong
and Lin, Chenhao
and Shen, Chao
and Wang, Qian",
title="Rethinking Adversarial Examples Exploiting Frequency-Based Analysis",
booktitle="Information and Communications Security",
  pages={73--89},
year="2021",}

@article{maiya2021frequency,
  title={A frequency perspective of adversarial robustness},
  author={Maiya, Shishira R and Ehrlich, Max and Agarwal, Vatsal and Lim, Ser-Nam and Goldstein, Tom and Shrivastava, Abhinav},
  journal={arXiv preprint arXiv:2111.00861},
  year={2021}
}

@inproceedings{maiya2023unifying,
  title={Unifying the Harmonic Analysis of Adversarial Attacks and Robustness.},
  author={Maiya, Shishira R and Ehrlich, Max and Agarwal, Vatsal and Lim, Ser-Nam and Goldstein, Tom and Shrivastava, Abhinav},
  booktitle={British Machine Vision Conference},
  year={2023},
  pages={620--621},
}

@inproceedings{
wang2023zeroshot,
title={Zero-Shot Image Restoration Using Denoising Diffusion Null-Space Model},
author={Yinhuai Wang and Jiwen Yu and Jian Zhang},
booktitle={International Conference on Learning Representations },
year={2023},
}

@inproceedings{vgg19,
  author       = {Karen Simonyan and
                  Andrew Zisserman},
  title        = {Very Deep Convolutional Networks for Large-Scale Image Recognition},
  booktitle    = {International Conference on Learning Representations},
  year         = {2015},
}

@inproceedings{
vit,
title={An Image is Worth 16x16 Words: Transformers for Image Recognition at Scale},
author={Alexey Dosovitskiy and Lucas Beyer and Alexander Kolesnikov and Dirk Weissenborn and Xiaohua Zhai and Thomas Unterthiner and Mostafa Dehghani and Matthias Minderer and Georg Heigold and Sylvain Gelly and Jakob Uszkoreit and Neil Houlsby},
booktitle={International Conference on Learning Representations},
year={2021},
}

@inproceedings{huang2017densely,
  title={Densely connected convolutional networks},
  author={Huang, Gao and Liu, Zhuang and Van Der Maaten, Laurens and Weinberger, Kilian Q},
  booktitle={IEEE/CVF Conference on Computer Vision and Pattern Recognition},
  year={2017},
  pages={4700--4708},
}

@inproceedings{liu2022convnet,
  title={A convnet for the 2020s},
  author={Liu, Zhuang and Mao, Hanzi and Wu, Chao-Yuan and Feichtenhofer, Christoph and Darrell, Trevor and Xie, Saining},
  booktitle={IEEE/CVF Computer Vision and Pattern Recognition},
  year={2022},
  pages={11976--11986},
}

@inproceedings{sa1,
  title={Obfuscated gradients give a false sense of security: Circumventing defenses to adversarial examples},
  author={Athalye, Anish and Carlini, Nicholas and Wagner, David},
  booktitle={International Conference on Machine Learning},
  year={2018},
  pages={274--283},
}

@article{sa2,
  title={On adaptive attacks to adversarial example defenses},
  author={Tramer, Florian and Carlini, Nicholas and Brendel, Wieland and Madry, Aleksander},
  journal={Annual Conference on Neural Information Processing Systems},
  year={2020},
  pages={1633--1645},
}

@inproceedings{zhuang2020adaptive,
  title={Adaptive checkpoint adjoint method for gradient estimation in neural ode},
  author={Zhuang, Juntang and Dvornek, Nicha and Li, Xiaoxiao and Tatikonda, Sekhar and Papademetris, Xenophon and Duncan, James},
  booktitle={International Conference on Machine Learning},
  year={2020},
  pages={11639--11649}
}

@INPROCEEDINGS{dino,
  author={Caron, Mathilde and Touvron, Hugo and Misra, Ishan and Jegou, Hervé and Mairal, Julien and Bojanowski, Piotr and Joulin, Armand},
  booktitle={IEEE/CVF International Conference on Computer Vision}, 
  title={Emerging Properties in Self-Supervised Vision Transformers}, 
  year={2021},
  pages={9650--9660},
}

@article{dinov2,
  author={Maxime Oquab and Timothée Darcet and Théo Moutakanni and Huy V. Vo and Marc Szafraniec and Vasil Khalidov and Pierre Fernandez and Daniel Haziza and Francisco Massa and Alaaeldin El-Nouby and Mido Assran and Nicolas Ballas and Wojciech Galuba and Russell Howes and Po-Yao Huang and Shang-Wen Li and Ishan Misra and Michael Rabbat and Vasu Sharma and Gabriel Synnaeve and Hu Xu and Hervé Jégou and Julien Mairal and Patrick Labatut and Armand Joulin and Piotr Bojanowski},
  title={DINOv2: Learning Robust Visual Features without Supervision},
  year={2024},
  journal={Transactions on Machine Learning Research},
}

@inproceedings{clip,
  title={Learning transferable visual models from natural language supervision},
  author={Radford, Alec and Kim, Jong Wook and Hallacy, Chris and Ramesh, Aditya and Goh, Gabriel and Agarwal, Sandhini and Sastry, Girish and Askell, Amanda and Mishkin, Pamela and Clark, Jack and others},
  booktitle={International Conference on Machine Learning},
  year={2021},
  pages={8748--8763},
}

@article{richards2013discrete,
  title={The discrete-time Fourier transform and discrete Fourier transform of windowed stationary white noise},
  author={Richards, Mark A},
  journal={Georgia Institute of Technology, Tech. Rep},
  year={2013}
}
\bibliographystyle{icml2025}

\newpage
\appendix
\onecolumn
\section{More Experimental Results}


\subsection{Hyperparameter Sensitivity Analysis}
In this section, we conduct a hyperparameter sensitivity analysis, and our method includes three hyperparameters: one for controlling the retention of the amplitude spectrum $D_A$, and the other two for controlling the retention of the phase spectrum $D_P$ and the projection range $\delta$. WideResNet28-10 serves as the classifier and we use PGD as the attack method with $\ell_{\infty}$ and $\epsilon=8/255$. From Figure \ref{hyper}, we can see that the best performance is achieved when $D_A=3$ and $D_P=2, \delta=0.2$.
\begin{figure*}[ht]
    \centering
    \includegraphics[width=0.7\textwidth]{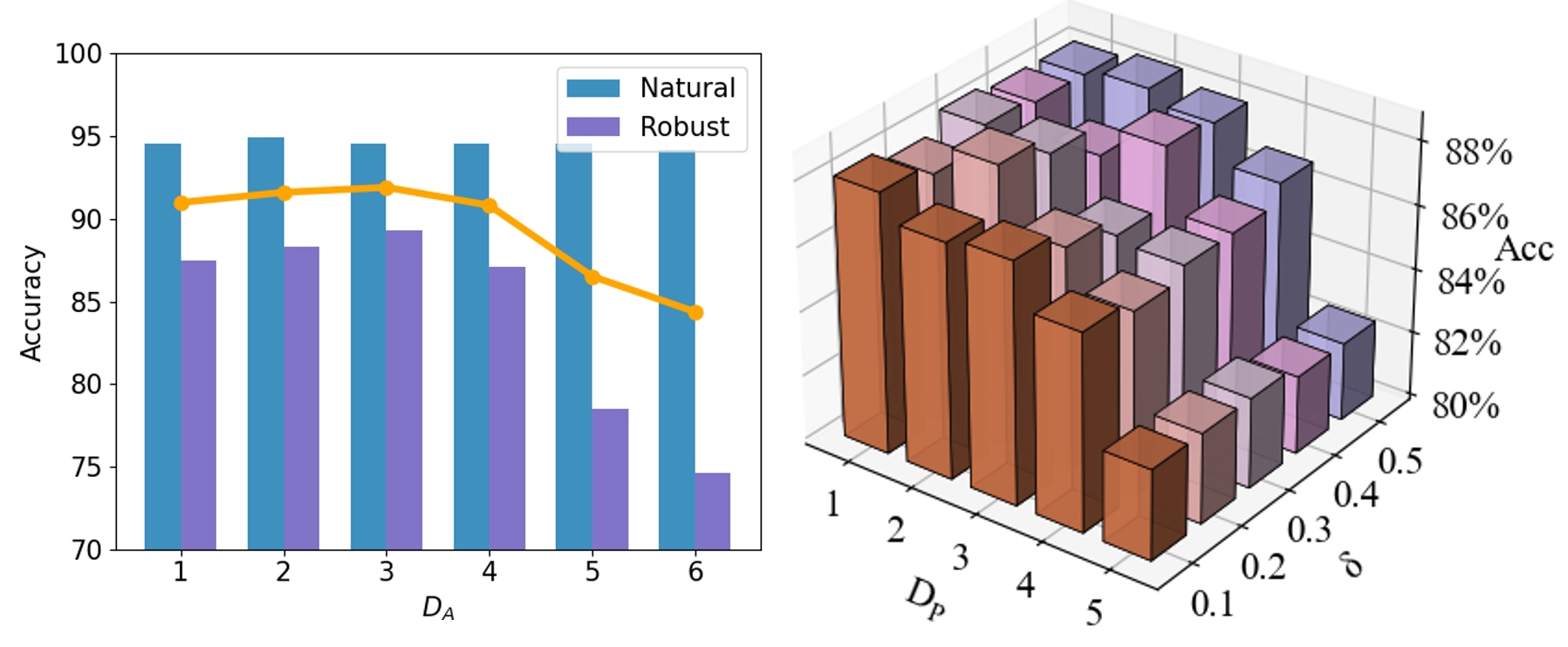}
    \caption{Robust and Standard Accuracy under different thresholds $D_A$  (left) and under different combinations of  $D_P$ and $\delta$ (right).}
    \label{hyper}
\end{figure*}
\subsection{Adaptive Attack and Surrogate Process}
Strong adaptive attacks \cite{sa1,sa2} require computing full gradients of diffusion-based adversarial purification methods. \cite{diffpure}  proposes to use the
adjoint method to compute full gradients of the
reverse generative process. The adjoint method
can compute the exact gradient in theory, but in practice, the adjoint relies on the performance of the numerical solver, whose performance becomes problematic in some cases as reported by \cite{zhuang2020adaptive}. Furthermore, the experiments conducted by \cite{robust-evaluation} reveal that this method tends to overestimate the robustness of the defensive measures. As suggested in \cite{robust-evaluation}, they use the approximated gradient obtained from a surrogate process. The surrogate process utilizes the fact that given the total amount of noise, we can denoise the same amount of noise with different numbers of denoising steps. Therefore, instead of using the entire denoising steps, we can mimic the original denoising process with fewer function calls, whose gradients can be obtained by back-propagating the forward and denosing process directly. To investigate whether our defense method is sensitive to the varying number of denoising steps in the surrogate process, we conducted an experimental analysis. 
\section{Discrete Fourier Transform}
\subsection{Preliminary}
\label{FFT}
Given an image $\mathbf{x}\in \mathbb{R}^{H\times W}$, we perform a two-dimensional discrete Fourier transform (DFT) on it:
\begin{equation}
\begin{aligned}
    \mathbf{x}(u,v)=DFT(\mathbf{x}(x,y))=\sum_{x=0}^{H-1}\sum_{y=0}^{W-1}\mathbf{x}(x,y)e^{j2\pi(\frac{ux}{H}+\frac{vy}{W})},
\end{aligned}
\end{equation}
where $u=0,1,2,...,H-1$ and $v=0,1,2,...,W-1$.\\
When the transform $\mathbf{x}(u,v)$ is known, $\mathbf{x}(x,y)$ can be obtained using the inverse discrete Fourier transform (IDFT):
\begin{equation}
    \mathbf{x}(x,y) = IDFT(\mathbf{x}(u,v))=\frac{1}{HW}\sum_{u=0}^{H-1}\sum_{v=0}^{W-1}\mathbf{X}(u,v)e^{j2\pi(\frac{ux}{H}+\frac{vy}{W})},
\end{equation}
where  $x=0,1,2,...,H-1$ and $y=0,1,2,...,W-1$.\\
Since the two-dimensional discrete Fourier transform is usually a complex function, it can be represented in polar coordinates:
\begin{equation}
    \mathbf{x}(u,v)=R(u,v)+jI(u,v)=|\mathbf{x}(u,v)|e^{j\phi(u,v)}.
\end{equation}
The method for calculating the amplitude is as follows:
\begin{equation}
    |\mathbf{x}(u,v)|=\sqrt{R^2(u,v)+I^2(u,v)}.
\end{equation}
The method for calculating the phase is as follows:
\begin{equation}
    \phi(u,v)=\arctan(\frac{I(u,v)}{R(u,v)}).
\end{equation}
\subsection{Property}
(Linearity) For any two images $\mathbf{x}$ and 
$\mathbf{y}$, where $a$ and $b$ are constants, then:
\begin{equation}
\begin{aligned}
    DFT(a\cdot \mathbf{x}(x,y)+b\cdot \mathbf{y}(x,y)) 
    &=DFT(a\cdot \mathbf{x}(x,y))+ DFT(b\cdot \mathbf{y}(x,y))\\
    &=a\cdot DFT(\mathbf{x}(x,y))+b\cdot  DFT(\mathbf{y}(x,y))\\
    &= a \cdot \mathbf{x}(u,v) + b \cdot \mathbf{y}(u,v).
\end{aligned}
\end{equation}
\section{Distribution of Adversarial Perturbations in the Frequency Domain}
\label{More}
\subsection{Experiment Settings}
We decompose the images into amplitude spectrum and phase spectrum using the discrete Fourier transform, exploring how adversarial perturbations affect the amplitude spectrum and phase spectrum. Here, we randomly selected 512 images from the ImageNet dataset. We tested different attack methods, including AutoAttack \cite{autoattack}, Projected Gradient Descent (PGD \cite{PGD}) under $\ell_\infty$ and $\ell_2$ nrom, as well as various perturbation radii, using different models including ResNet50 \cite{resnet}, VGG19  \cite{vgg19}, ViT \ref{vit} \cite{vit}, DenseNet \ref{densenet} \cite{huang2017densely} and ConvNeXT \ref{convnext} \cite{liu2022convnet}.\\
Given two images, one normal image $\mathbf{x}$ and one adversarial image $\mathbf{x}_{adv}$, we decompose them into amplitude spectrum and phase spectrum using the discrete Fourier transform as follows:
\begin{equation}
\begin{aligned}
    \mathbf{x}(u,v)&=DFT(\mathbf{x})=|\mathbf{x}(u,v)|e^{i\phi_{\mathbf{x}}(u,v)},\\
    \mathbf{x}_{adv}(u,v)&=DFT(\mathbf{x}_{adv})=|\mathbf{x}_{adv}(u,v)|e^{i\phi_{\mathbf{x}_{adv}}(u,v)}.
\end{aligned}  
\end{equation}
To investigate the variation of adversarial perturbations with frequency, we calculate the differences of the  amplitude spectrum and phase spectrum between the normal image and the adversarial image.\\
For the  amplitude spectrum, the amplitude spectrum of the image exhibits low-pass characteristics with respect to frequency, specifically:
\begin{equation}
    |\mathbf{x}(u,v)|\propto D(u,v)^{-\alpha},
\end{equation}
Typically, the parameter $\alpha$ is 2 or 3. Due to the power-law distribution characteristic of the  amplitude spectrum, we choose to quantify the differences between the  amplitude spectra of adversarial images and normal images using rhe absolute value of the percentage difference:
\begin{equation}
    \mathbb{E}(|\frac{|\mathbf{x}_{adv}(u,v)|-|\mathbf{x}(u,v)|}{|\mathbf{x}(u,v)|}|),
\end{equation}
The distribution of the phase spectrum is typically random and closely related to the specific content of the image. Its values range from 0 to 2$\pi$. Therefore, we choose to measure the differences between the phase spectra of adversarial images and nromal images using the absolute value of the differences:
\begin{equation}
    \mathbb{E}(|\phi_{\mathbf{x}}(u,v)-\phi_{\mathbf{x}_{adv}}(u,v)|).
\end{equation}
\subsection{More Experiment Results}
\begin{figure*}[ht]
    \centering
    \includegraphics[width=0.6\textwidth]{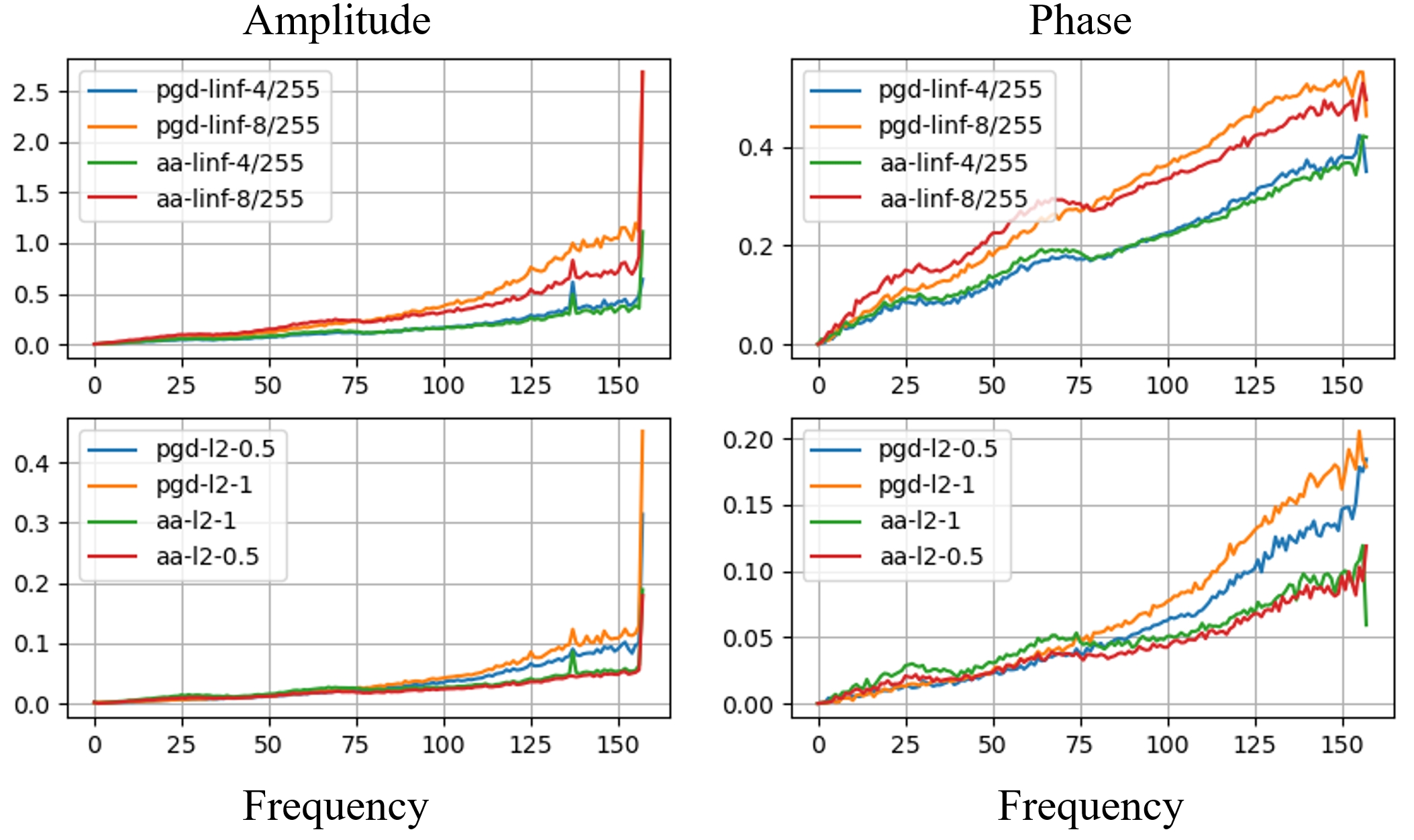}
    \caption{\textbf{Vision Transformer} serves as backbone. We decompose the image into the amplitude spectrum (left) and the phase spectrum (right), and calculate the differences between the adversarial images and the normal images, respectively.}
    \label{vit}
\end{figure*}
\begin{figure*}[ht]
    \centering
    \includegraphics[width=0.6\textwidth]{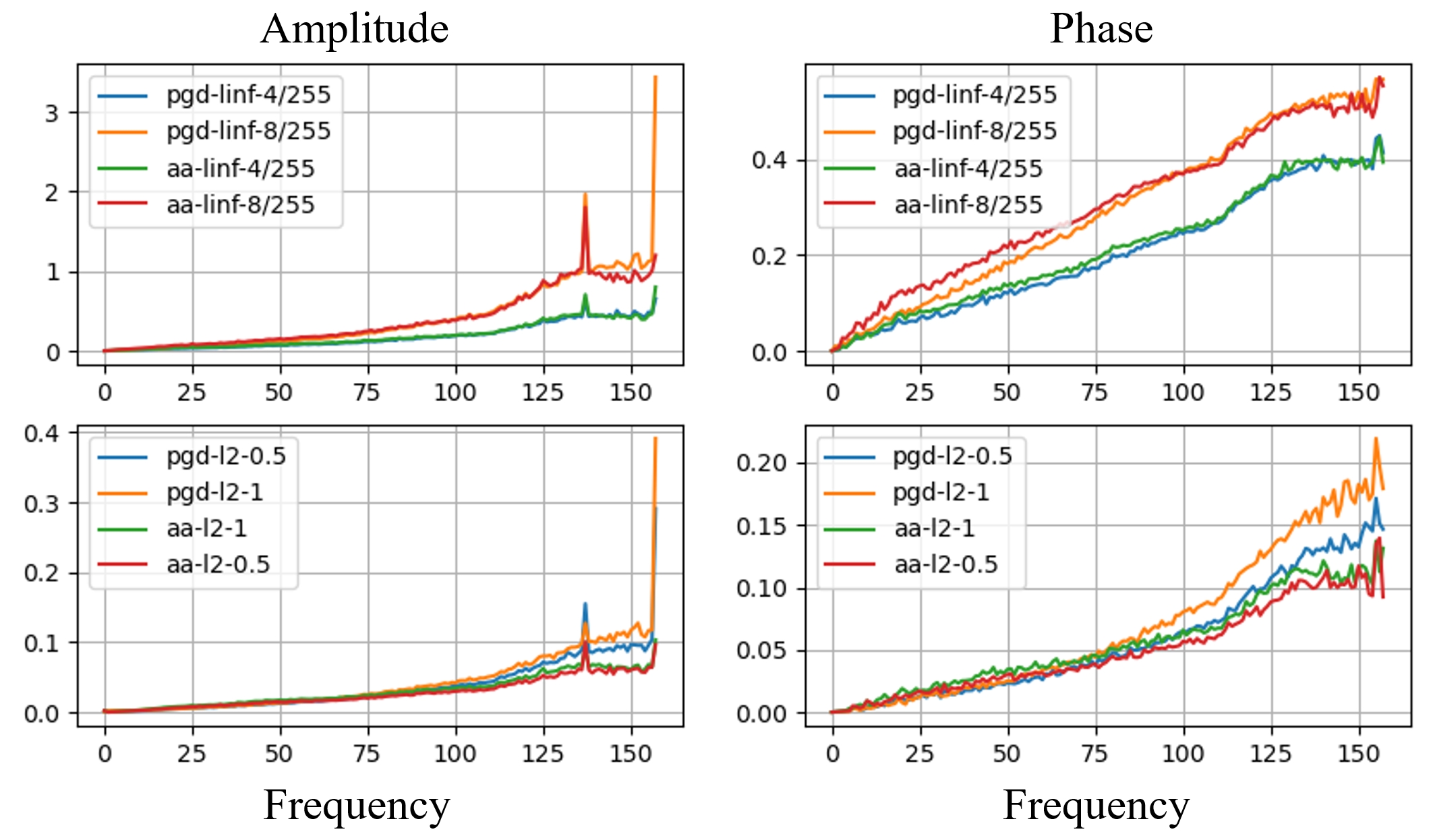}
    \caption{\textbf{DenseNet121} serves as backbone. We decompose the image into the amplitude spectrum (left) and the phase spectrum (right), and calculate the differences between the adversarial images and the normal images, respectively.}
    \label{densenet}
\end{figure*}
\begin{figure*}[ht]
    \centering
    \includegraphics[width=0.6\textwidth]{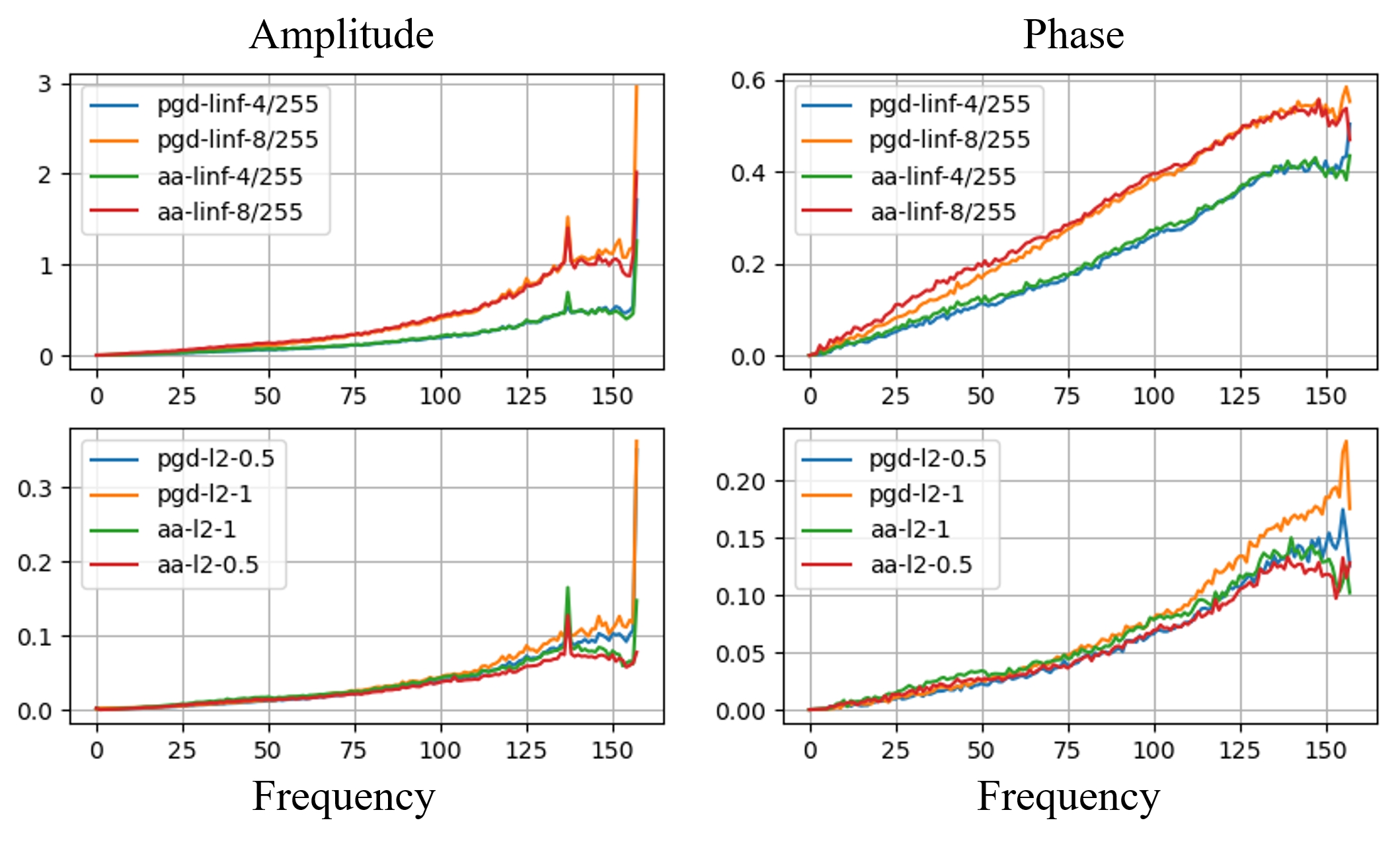}
    \caption{\textbf{ConvNeXT} serves as backbone. We decompose the image into the amplitude spectrum (left) and the phase spectrum (right), and calculate the differences between the adversarial images and the normal images, respectively.}
    \label{convnext}
\end{figure*}

\clearpage
\section{Proof}

\subsection{Proof of Theorem 3.2.}
\label{proof-3-2}

\begin{theorem} (\textbf{Modified Edition})
The variance of the difference of amplitude at time-step $t$ between clean image $\mathbf{x}_0$ and noisy image $\mathbf{x}_t$ at arbitrary coordinates $(u,v)$ at frequency domain is as follows:
    \begin{equation}
    \begin{aligned}
        Var(\Delta A_t(u,v))\approx\frac{1-\overline{\alpha}_t}{2} -\frac{(1-\overline{\alpha}_t)^2}{16|\mathbf{x}_0(u,v)|^2\overline{\alpha}_t}.
    \end{aligned}
    \end{equation}
    The RHS is monotonically increasing with respect to $t$, This means that as $t$ increases, the amplitude spectrum of the original image at arbitrary coordinate $(u,v)$  is increasingly disrupted by noise.
\end{theorem}
\begin{proof}
The forward process of DDPM \cite{DDPM} in the time domain is as follows:
\begin{equation}
\label{forward}
     \mathbf{x}_t = \sqrt{\overline{\alpha}_t}\mathbf{x}_0+\sqrt{1-\overline{\alpha}_t}\mathbf{\mathbf{\epsilon}},
\end{equation}
where $t$ denotes the time-step.  $\mathbf{x}_0$ denotes the input clean image. $\mathbf{\mathbf{\epsilon}}\sim\mathcal{N}(0,\mathbf{I})$ denotes the gaussian noise. To analyze the impact of noise on different frequency components, we transform the \eqref{forward} from time domain into the frequency domain via  discrete Fourier transform  as follows:
\begin{equation}
\begin{aligned}
    DFT(\mathbf{x}_t)&=DFT (\sqrt{\overline{\alpha}_t}\mathbf{x}_0+\sqrt{1-\overline{\alpha}_t}\mathbf{\mathbf{\epsilon}})\\
    &=\sqrt{\overline{\alpha}_t}DFT(\mathbf{x}_0)+\sqrt{1-\overline{\alpha}_t}DFT(\mathbf{\mathbf{\epsilon}}),\\
\end{aligned}
\end{equation}
To simplify the equation. We omit the DFT and use $\mathbf{x}_0(u,v)$, $\mathbf{x}_t(u,v)$ and $\mathbf{\epsilon}(u,v)$ at coordinate $(u,v)$ at frequency daomain:
\begin{equation}
    \mathbf{x}_t(u,v) = \sqrt{\overline{\alpha}_t}\mathbf{x}_0(u,v)+\sqrt{1-\overline{\alpha}_t}\mathbf{\epsilon}(u,v),
\end{equation}
To demonstrate the forward process will damage the amplitude spectrum, in the following we give the corresponding proof:
    \begin{equation}
    \begin{aligned}
        \mathbb{E}(|\mathbf{x}_t(u,v)|^2)&=\mathbb{E}(|\sqrt{\overline{\alpha}_t}\mathbf{x}_0(u,v)+\sqrt{1-\overline{\alpha}_t}\mathbf{\epsilon}(u,v)|^2)\\
        &=\mathbb{E}(|\sqrt{\overline{\alpha}_t}\mathbf{x}_0(u,v)|^2+|\sqrt{1-\overline{\alpha}_t}\mathbf{\epsilon}(u,v)|^2+2\mathfrak{R}\mathfrak{e}\{\sqrt{\overline{\alpha}_t}\mathbf{x}_0(u,v)*\sqrt{1-\overline{\alpha}_t}\mathbf{\epsilon}(u,v)^*\})\\
        &=\mathbb{E}(\overline{\alpha}_t |\mathbf{x}_0(u,v)|^2+(1-\overline{\alpha}_t)|\mathbf{\epsilon}(u,v)|^2+2\sqrt{\overline{\alpha}_t}\sqrt{1-\overline{\alpha}_t}\mathfrak{R}\mathfrak{e}\{\mathbf{x}_0(u,v)*\mathbf{\epsilon}(u,v)^*\}),\\
    \end{aligned}
    \end{equation}
    where, $\mathfrak{R}\mathfrak{e}$ represents the real part and $\mathbf{\epsilon}(u,v)^*$  is the complex conjugate of $\mathbf{\epsilon}(u,v)$.
    The power spectral density of the noise is flat, and the noise is independent at different frequencies. Its mean is 0, and the variance is a constant (When the distribution follows a standard normal distribution, the variance is 1.):
    \begin{equation}
    \begin{aligned}
        &\mathbb{E}(\mathbf{\epsilon}(u,v))=0\\
        &\mathbb{E}(|\mathbf{\epsilon}(u,v)|^2)=\sigma^2=1
    \end{aligned}    
    \end{equation}
    Therefore, 
    \begin{equation}
        \begin{aligned}
            \mathbb{E}(|\mathbf{x}_t(u,v)|^2)&=\mathbb{E}(\overline{\alpha}_t |\mathbf{x}_0(u,v)|^2+(1-\overline{\alpha}_t)|\mathbf{\epsilon}(u,v)|^2+2\sqrt{\overline{\alpha}_t}\sqrt{1-\overline{\alpha}_t}\mathfrak{R}\mathfrak{e}\{\mathbf{x}_0(u,v)*\mathbf{\epsilon}(u,v)^*\})\\
            &=\overline{\alpha}_t |\mathbf{x}_0(u,v)|^2+(1-\overline{\alpha}_t)\mathbb{E}(|\mathbf{\epsilon}(u,v)|^2)+2\sqrt{\overline{\alpha}_t}\sqrt{1-\overline{\alpha}_t}\mathfrak{R}\mathfrak{e}\{\mathbf{x}_0(u,v)*\mathbb{E}(\mathbf{\epsilon}(u,v)^*)\}\\
            &=\overline{\alpha}_t |\mathbf{x}_0(u,v)|^2+(1-\overline{\alpha}_t),
        \end{aligned}
    \end{equation}
    \begin{equation}
    \begin{aligned}
    \mathbf{x}_t(u,v) =& \sqrt{\overline{\alpha}_t}\mathbf{x}_0(u,v)+\sqrt{1-\overline{\alpha}_t}\mathbf{\epsilon}(u,v)\\
    &=\underbrace{\mathfrak{R}\mathfrak{e}(\sqrt{\overline{\alpha}_t}\mathbf{x}_0(u,v))+\mathfrak{R}\mathfrak{e}(\sqrt{1-\overline{\alpha}_t}\mathbf{\epsilon}(u,v))}_{\sim\mathcal{N}(\mathfrak{R}\mathfrak{e}(\sqrt{\overline{\alpha}_t}\mathbf{x}_0(u,v)),\frac{1-\overline{\alpha}_t}{2})}+i\underbrace{(\mathfrak{I}\mathfrak{m}(\sqrt{\overline{\alpha}_t}\mathbf{x}_0(u,v))+\mathfrak{I}\mathfrak{m}(\sqrt{1-\overline{\alpha}_t}\mathbf{\epsilon}(u,v)))}_{\sim\mathcal{N}(\mathfrak{I}\mathfrak{m}(\sqrt{\overline{\alpha}_t}\mathbf{x}_0(u,v)),\frac{1-\overline{\alpha}_t}{2})},
    \end{aligned}
    \end{equation}
We can see that the means of the real part and the imaginary part are different, the variances are the same, and they are independent of each other \cite{richards2013discrete}.  Therefore, the amplitude $|\mathbf{x}_t(u,v)|$ follows a Rice distribution:
\begin{equation}
    f(|\mathbf{x}_t(u,v)|) = \frac{|\mathbf{x}_t(u,v)|}{\sigma^2} \exp\left(-\frac{|\mathbf{x}_t(u,v)|^2 + \nu^2}{2\sigma^2}\right) I_0\left(\frac{|\mathbf{x}_t(u,v)|\nu}{\sigma^2}\right),
\end{equation}
With the assumption of $SNR_t\gg1$ , we have:
\begin{equation}
    \mathbb{E}(|\mathbf{x}_t(u,v)|)\approx \nu\approx\nu+\frac{\sigma^2}{2\nu}=\sqrt{\overline{\alpha}_t}|\mathbf{x}_0(u,v)| + \frac{1-\overline{\alpha}_t}{4\sqrt{\overline{\alpha}_t}|\mathbf{x}_0(u,v)|},
\end{equation}
The variance of the difference of amplitude at time-step $t$ between clean image $\mathbf{x}_0$ and noisy image $\mathbf{x}_t$ is as follows:
\begin{equation}
    \begin{aligned}
        Var(\Delta A_t(u,v))&=Var(|\mathbf{x}_t(u,v)|-|\mathbf{x}_0(u,v)|)\\
        &=Var(|\mathbf{x}_t(u,v)|)\\
        &=\mathbb{E}(|\mathbf{x}_t(u,v)|^2)-\mathbb{E}^2(|\mathbf{x}_t(u,v)|)\\
        &\approx \overline{\alpha}_t |\mathbf{x}_0(u,v)|^2+(1-\overline{\alpha}_t) - (\sqrt{\overline{\alpha}_t}|\mathbf{x}_0(u,v)| + \frac{1-\overline{\alpha}_t}{4\sqrt{\overline{\alpha}_t}|\mathbf{x}_0(u,v)|})^2\\
        &=\frac{1-\overline{\alpha}_t}{2} -\frac{(1-\overline{\alpha}_t)^2}{16|\mathbf{x}_0(u,v)|^2\overline{\alpha}_t},
    \end{aligned}
\end{equation}
When $SNR_t$ is sufficiently large, it is clearly that $ Var(\Delta A_t(u,v))$  monotonically decreasing with $t$.
\subsection{Proof of Theorem 3.4.}
\label{proof-3-4}

\begin{theorem}
According to Theorem 3.2. in \cite{diffpure} that $t$ should be sufficiently small. Therefore, we assume $SNR_t>1$. the variance of difference of phase  $Var(\Delta\theta(u,v))$ between input image $\mathbf{x}_0$ and noisy image $\mathbf{x}_t$ at arbitrary coordinates $(u,v)$ at frequency domain is as follows:
    \begin{equation}
         Var(\Delta\theta_t(u,v))\approx\frac{1}{\sqrt{1-\frac{1}{SNR_t^2}}}-1,
    \end{equation}
    where signal to noise ratio (SNR) at time-step $t$ is defined as follows:
    \begin{equation}
        SNR_t(u,v)=\frac{\sqrt{\overline{\alpha}_t}|\mathbf{x}_0(u,v)|}{\sqrt{1-\overline{\alpha}_t}|\mathbf{\epsilon}(u,v)|}.
    \end{equation}
\end{theorem}

\begin{proof}
The forward process of DDPM \cite{DDPM} in the time domain is as follows:
\begin{equation}
     \mathbf{x}_t = \sqrt{\overline{\alpha}_t}\mathbf{x}_0+\sqrt{1-\overline{\alpha}_t}\mathbf{\mathbf{\epsilon}}.
\end{equation}
Where $t$ denotes the time-step.  $\mathbf{x}_0$ denotes the input clean image. $\mathbf{\mathbf{\epsilon}}\sim\mathcal{N}(0,\mathbf{I})$ denotes the gaussian noise. To analyze the impact of noise on different frequency components, we transform the \eqref{forward} from time domain into the frequency domain via  discrete Fourier transform  as follows:
\begin{equation}
\begin{aligned}
    DFT(\mathbf{x}_t)&=DFT (\sqrt{\overline{\alpha}_t}\mathbf{x}_0+\sqrt{1-\overline{\alpha}_t}\mathbf{\mathbf{\epsilon}})\\
    &=\sqrt{\overline{\alpha}_t}DFT(\mathbf{x}_0)+\sqrt{1-\overline{\alpha}_t}DFT(\mathbf{\mathbf{\epsilon}}),\\
\end{aligned}
\end{equation}
To simplify the equation. We omit the DFT and use $\mathbf{x}_0(u,v)$, $\mathbf{x}_t(u,v)$ and $\mathbf{\epsilon}(u,v)$ at coordinate $(u,v)$ at frequency daomain:
\begin{equation}
    \mathbf{x}_t(u,v) = \sqrt{\overline{\alpha}_t}\mathbf{x}_0(u,v)+\sqrt{1-\overline{\alpha}_t}\mathbf{\epsilon}(u,v),
\end{equation}
To demonstrate the forward process will damage the phase spectrum, in the following we give the corresponding proof: 
\begin{equation}
\begin{aligned}
    \mathbf{x}_t(u,v) &= \sqrt{\overline{\alpha}_t}\mathbf{x}_0(u,v)+\sqrt{1-\overline{\alpha}_t}\mathbf{\epsilon}(u,v)\\
    &=\sqrt{\overline{\alpha}_t}|\mathbf{x}_0(u,v)|e^{i\phi_{\mathbf{x}_0}(u,v)}+\sqrt{1-\overline{\alpha}_t}|\mathbf{\epsilon}(u,v)|e^{i\phi_{\mathbf{\epsilon}}(u,v)}\\
    &=S_te^{i\phi_{\mathbf{x}_0}(u,v)}+N_te^{i\phi_{\mathbf{\epsilon}}(u,v)}\\
    &=S_te^{i\phi_{\mathbf{x}_0}(u,v)}(1+\frac{N_t}{S_t}e^{i(\phi_{\mathbf{\epsilon}}(u,v)-\phi_{\mathbf{x}_0}(u,v))})\\
    &=S_te^{i\phi_{\mathbf{x}_0}(u,v)}(1+K_te^{i\phi}).\\
\end{aligned}
\end{equation}
Here we can get the difference between the phase of the original image $\mathbf{x}_0$ and the noisy image $\mathbf{x}_t$:
\begin{equation}
\begin{aligned}
    \Delta\theta&=\phi_{\mathbf{x}_t}(u,v)-\phi_{\mathbf{x}_0}(u,v)\\
    &=\arg(1+K_te^{i\phi})\\
    &=\arg(1+K_t(cos(\phi)+isin(\phi)))\\    &=\arg(\underbrace{1+K_t\cos(\phi)}_{Real}+\underbrace{iK_tsin(\phi))}_{Imaginary}\\
    &=\arctan(\frac{K_tsin(\phi)}{1+K_t\cos(\phi)})\\
    &\approx \frac{K_tsin(\phi)}{1+K_t\cos(\phi)}.
\end{aligned}
\end{equation}
The last line of the formula is obtained through a first-order Taylor expansion\\
The phase spectrum of Gaussian noise $\mathbf{\epsilon}$ is uniformly distributed in the range [0,2$\pi$]:
\begin{equation}
    p(\phi_{\mathbf{\epsilon}}) = \begin{cases}   
\frac{1}{2\pi} & \text{if } 0 \leq \phi_{\mathbf{\epsilon}} < 2\pi \\
0 & \text{otherwise}  
\end{cases}  .
\end{equation}
Due to the periodicity of phase, the range of $\phi_\mathbf{\epsilon}(u,v)-\phi_{\mathbf{x}_0}(u,v)$ is also uniformly distributed in the range [0,2$\pi$].\\
The expectation of the phase difference is as follows:
\begin{equation}
\begin{aligned}
    E( \Delta\theta)&=E(\frac{K_tsin(\phi)}{1+K_t\cos(\phi)})\\
    &=\frac{1}{2\pi}\int_{0}^{2\pi}\frac{K_tsin(\phi)}{1+K_t\cos(\phi)}d\phi\\
    &=0.
\end{aligned}
\end{equation}
The Variance of the phase difference is as follows:

\begin{equation}
\begin{aligned}
    Var( \Delta\theta)&=E( (\Delta\theta)^2)-(E(\Delta\theta))^2=E( (\Delta\theta)^2)\\
    &=\frac{1}{2\pi}\int_0^{2\pi}\frac{K_t^2\sin^2(\phi)}{(1+K_t\cos(\phi))^2}d\phi\\
    &=\frac{1}{2\pi}\int_0^{2\pi}K_t\sin(\phi)d(\frac{1}{1+K_tcos(\phi)})\\
    &=\frac{1}{2\pi}(\underbrace{\frac{K_t\sin(\phi)}{1+K_t\cos(\phi)}\big|_{0}^{2\pi}}_{0}-\int_{0}^{2\pi}\frac{K_tcos(\phi)}{1+K_t\cos(\phi)}d\phi)\\
    &=-\frac{1}{2\pi}\int_{0}^{2\pi}1-\frac{1}{1+K_t\cos(\phi)}d\phi\\
    &=\frac{1}{2\pi}(\int_{0}^{2\pi}\frac{1}{1+K_t\cos(\phi)}d\phi-2\pi)\\
    &=\frac{1}{2\pi}(\int_{-\infty}^{+\infty}\frac{1}{1+K_t\frac{1-t^2}{1+t^2}}\frac{2dt}{1+t^2}-2\pi)\\
    &=\frac{1}{2\pi}(\int_{-\infty}^{+\infty}\frac{2dt}{1+K_t+(1-K_t)t^2}-2\pi)\\
    &=\frac{1}{2\pi}(\frac{1}{\sqrt{1-K_t^2}}\arctan t\big|_{-\infty}^{+\infty}-2\pi)\\
    &=\frac{1}{\sqrt{1-K_t^2}}-1.
\end{aligned}
\end{equation}
\end{proof}
\end{proof}
\section{Visualization}
\subsection{ImageNet}
We randomly select some images for visualization and chose DiffPure \cite{diffpure} and REAP \cite{robust-evaluation} as the baselines. The attack method we use here is PGD. And to make the visualization effect more pronounced, with the perturbation radius set to $\ell_{\infty}=12/255$. We plot the original images, the images after adversarial attacks, and the images purified using our method and other methods.
\begin{figure*}[htbp]
    \centering
    \includegraphics[width=1.0\textwidth]{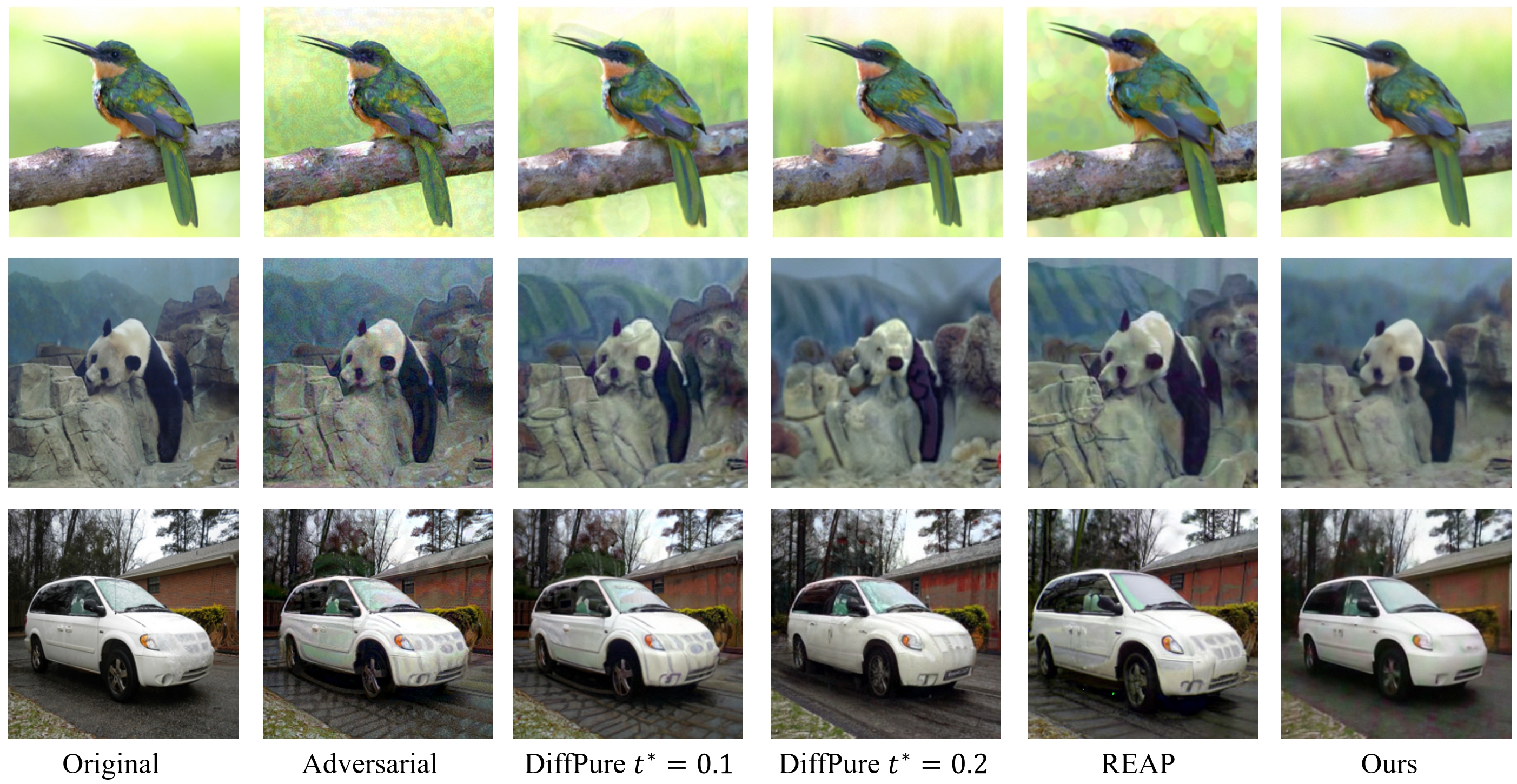}
    \caption{Visualization of some randomly selected images from ImageNet dataset.}
\end{figure*}
\begin{figure*}[htbp]
    \centering
    \includegraphics[width=1.0\textwidth]{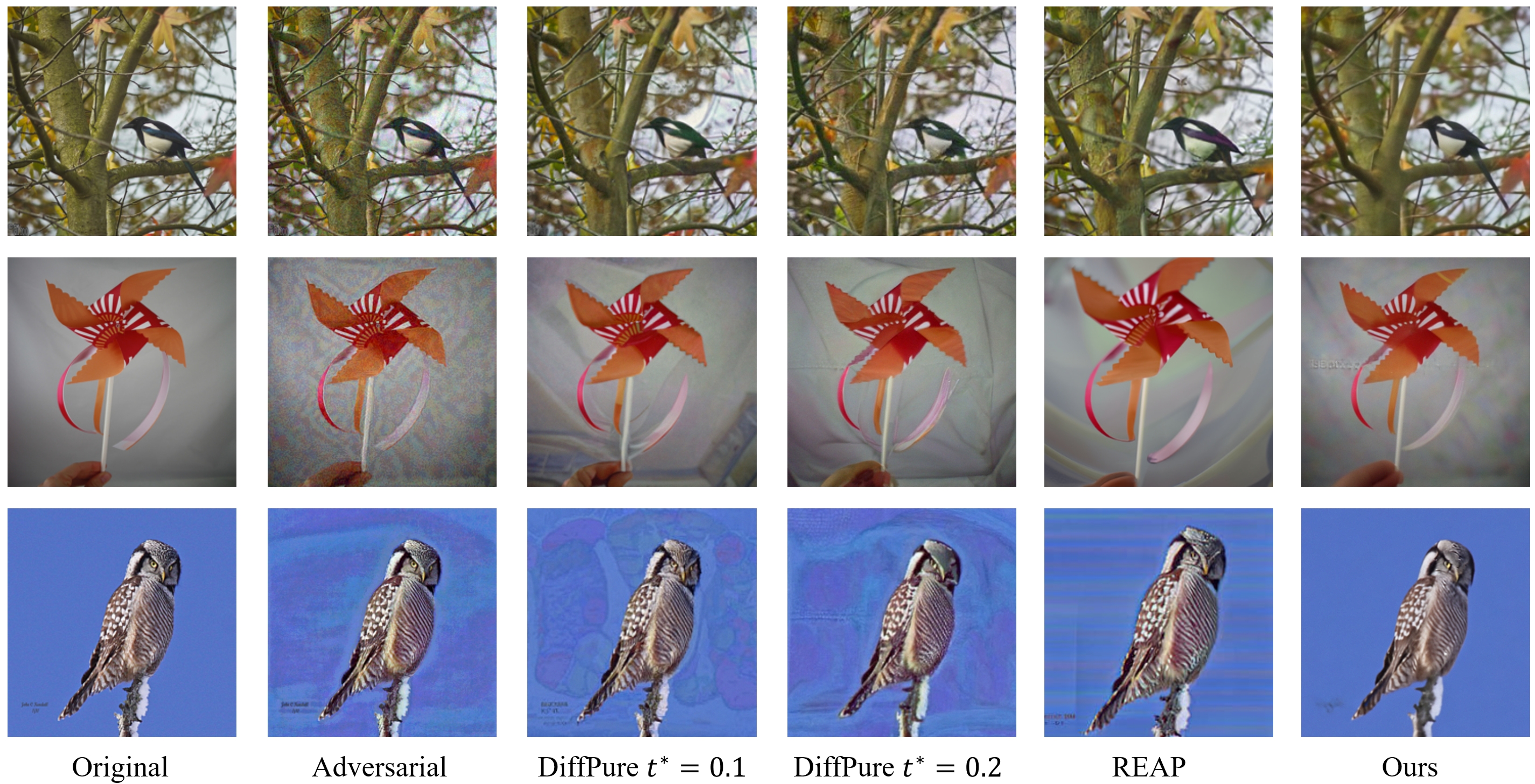}
    \caption{Visualization of some randomly selected images from ImageNet dataset.}
\end{figure*}

\subsection{CIFAR10}
We randomly select 64 images for visualization, choosing PGD as the attack method and setting the attack radius to $\ell_{\infty}=8/255$.
\begin{figure*}[h]
    \centering
    \includegraphics[width=1.0\textwidth]{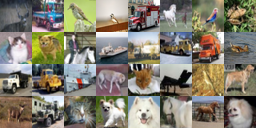}
    \caption{Visualization of original images randomly selected from CIFAR10 dataset.}
\end{figure*}
\begin{figure*}[h]
    \centering
    \includegraphics[width=1.0\textwidth]{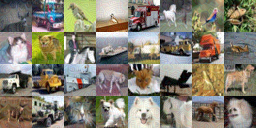}
    \caption{Visualization of adversarial images randomly selected from CIFAR10 dataset.}
\end{figure*}
\begin{figure*}[h]
    \centering
    \includegraphics[width=1.0\textwidth]{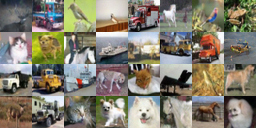}
    \caption{Visualization of purified images randomly selected from CIFAR10 dataset.}
\end{figure*}
\end{document}